%% file: scalable_vi.tex
\documentclass{article}

\PassOptionsToPackage{numbers, compress}{natbib}

\usepackage[final]{neurips_2021}

\usepackage[utf8]{inputenc} 
\usepackage[T1]{fontenc}    
\usepackage{hyperref}       
\usepackage{url}            
\usepackage{booktabs}       
\usepackage{amsfonts}       
\usepackage{nicefrac}       
\usepackage{microtype}      
\usepackage{xcolor}         
\usepackage{wrapfig} 
\usepackage{amsthm}
\usepackage{relsize}

\usepackage{graphicx}
\usepackage{pdflscape} 
\usepackage{subcaption}
\usepackage{algpseudocode,algorithm}
\usepackage{paralist}
\usepackage{amsmath, bm}
\usepackage{xcolor}
\usepackage{nccmath}
\usepackage{wrapfig}
\usepackage{tabularx}
\usepackage{upgreek}
\usepackage{pdflscape}
\usepackage{paralist}
\usepackage{subcaption}

\usepackage{hyperref}
\usepackage[nameinlink]{cleveref} 
\usepackage{enumitem}

\usepackage{tikz}
\usetikzlibrary{bayesnet}

\title{Amortized Variational Inference in Simple Hierarchical Models}

\author{%
  Abhinav Agrawal \\
  College of Information and Computer Science\\
  Univeristy Of Massachusetts Amherst\\
  \texttt{aagrawal@cs.umass.edu} \\
  \And
  Justin Domke \\
  College of Information and Computer Science\\
  Univeristy Of Massachusetts Amherst\\
  \texttt{domke@cs.umass.edu} \\
}

\definecolor{mycol}{rgb}{0,0,0.65}
\hypersetup{
    colorlinks,
    linkcolor={mycol},
    citecolor={mycol},
    urlcolor={mycol}
}

\DeclareMathOperator*{\argmax}{arg\,max}
\DeclareMathOperator*{\argmin}{arg\,min}

\newtheoremstyle{exampstyle}
  {2.5\topsep} 
  {\topsep} 
  {} 
  {} 
  {\bfseries} 
  {.} 
  {.5em} 
  {} 

\theoremstyle{exampstyle}\newtheorem{thm}{Theorem}
\theoremstyle{exampstyle}\newtheorem{defn}[thm]{Definition}
\theoremstyle{exampstyle}
\theoremstyle{exampstyle}\newtheorem{cor}[thm]{Corrolary}
\theoremstyle{exampstyle}
\theoremstyle{exampstyle}\newtheorem{claim}[thm]{Claim}
\newtheorem*{thm*}{Theorem}
\newtheorem*{claim*}{Claim}
\crefname{defn}{definition}{definitions}
\Crefname{defn}{Definition}{Definitions}

\global\long\def\argmin{\operatornamewithlimits{argmin}}%

\global\long\def\argmax{\operatornamewithlimits{argmax}}%

\global\long\def\R{\mathbb{R}}%

\global\long\def\E{\operatornamewithlimits{\mathbb{E}}}%

\global\long\def\N{\mathcal{N}}%

\global\long\def\pars#1{\left(#1\right)}%

\global\long\def\pp#1{(#1)}%

\global\long\def\bracs#1{\left[#1\right]}%

\global\long\def\verts#1{\left\vert #1\right\vert }%

\global\long\def\KL#1{[#1]}%

\global\long\def\KL#1#2{KL\pars{#1\middle\Vert#2}}%

\global\long\def\ELBO#1#2{\mathcal{L}\pars{#1\middle\Vert#2}}%

\newcommand{\rz}{\mathsf{z}}
\newcommand{\rtheta}{\uptheta}
\newcommand{\qj}{\smash{q_\phi^\mathrm{\mathsmaller Joint}}}

\newcommand{\qb}{\smash{q_{v,w}^{\mathrm{\mathsmaller Branch}}}}
\newcommand{\qab}{\smash{q_{v,u}^{\mathrm{\mathsmaller Amort}}}}

\newcommand{\net}{{\normalfont \texttt{net}}}

\allowdisplaybreaks

\begin{document}

\maketitle

\input{./abstract.tex}
\input{./introduction.tex}
\input{./method.tex}

\section*{Acknowledgements and Disclosure of Funding}
This material is based upon work supported in part by the National Science Foundation under Grant No. 1908577.

\bibliography{scalable_vi}
\bibliographystyle{plainnat}

\newpage
\input{./appendix}

\end{document}

%% file: abstract.tex
\begin{abstract}

It is difficult to use subsampling with variational inference in hierarchical models since the number of local latent variables scales with the dataset. Thus, inference in hierarchical models remains a challenge at large scale. It is helpful to use a variational family with structure matching the posterior, but optimization is still slow due to the huge number of local distributions. Instead, this paper suggests an amortized approach where shared parameters simultaneously represent all local distributions. This approach is similarly accurate as using a given joint distribution (e.g., a full-rank Gaussian) but is feasible on datasets that are several orders of magnitude larger. It is also dramatically faster than using a structured variational distribution.
\end{abstract}


%% file: introduction.tex
\section{Introduction}



Hierarchical Bayesian models are a general framework where parameters 
of ``groups'' are drawn from some shared distribution, 
and then observed data is drawn from a distribution specified by each 
group's parameters. After data is observed, the inference problem is 
to infer both the parameters for each group and the shared parameters. 
These models have proven useful in various domains \cite{gelman2020most} including 
hierarchical regression amd classification \cite{gelman2006data}, 
topic models 
\cite{blei2003latent,lafferty2006correlated,blei2012probabilistic}, 
polling \cite{gelman2009red,lax2012democratic}, 
epidemiology \cite{lawson2008bayesian},
ecology \cite{cressie2009accounting}, 
psychology \cite{vallerand1997toward},
matrix-factorization \cite{tipping1999probabilistic}, 
and collaborative filtering \cite{lim2007variational,salakhutdinov2008bayesian}.

A proven technique for scaling variational inference (VI) to large datasets is
 subsampling. 
 The idea is that if the target model has the
 form $p(z,y)=p(z)\prod_i p(y_i \vert z)$ then an unbiased gradient can
 be estimated while only evaluating $p(z)$ and $p(y_i \vert z)$ 
 at a few $i$ \cite{hoffman2013stochastic,kingma2013auto,pmlr-v32-rezende14,ranganath14,titsias2014doubly,hoffman2015ssvi}. 

This paper addresses hierarchical models of the form $p(\theta,z,y)=p(\theta)\prod_i p(z_i, y_i \vert \theta)$, where only $y$ is observed. There are two challenges. First, the number of local latent variables $z_i$ increases with the dataset, meaning the posterior distribution increases in dimensionality. Second, there is often a dependence between $z_i$ and $\theta$ which must be captured to get strong results \cite{hoffman2015ssvi}.

The aim of this paper is to develop a black-box variational inference scheme that can scale to large hierarchical models without losing benefits of a joint approximation. 
Our solution takes three steps. First, in the true posterior, the different latent variables $z_i$ are conditionally independent given $\theta$, which suggests using a variational family of the same form. We confirm this intuition by showing that for any joint variational family $q(\theta,z)$, one can define a corresponding "branch" family $q(\theta)\prod_i q(z_i \vert \theta)$ such that inference will be equally accurate (\cref{thm: branch q better than q for hbd}). We call inference using such a family the "branch" approach.

Second, we observe that if using the branch approach, the optimal local variational parameters can be computed only from $\theta$ and local data (\cref{eq: optimal local wi}). Thus, we propose to amortize the computation of the local variational parameters by learning a network to approximately solve that optimization. We show that when the target distribution is symmetric over latent variables, this will be as accurate as the original joint family, assuming a sufficiently capable amortization network (\cref{claim: amortization}).

Third, we note that in many real hierarchical models, there are many i.i.d. data generated from each local latent variable. This presents a challenge for learning an amortization network, since the full network should deal with different numbers of data points and naturally reflect the symmetry between the inputs (that is, without having to relearn the symmetry.) We propose an approach where a preliminary "feature" network processes each datum, after which they are combined with a pooling operation which forms the input for a standard network (\cref{sec: locally iid}). This is closely related to the "deep sets" \cite{deepsets} strategy for permutation invariance. 

We validate these methods on a synthetic model where exact inference is possible, and on a user-preference model for the MovieLens dataset with 162K users who make 25M ratings of different movies. At small scale (2.5K ratings), we show similar accuracy using a dense joint Gaussian, a branch distribution, or our amortized approach. At moderate scale (180K ratings), joint inference is intractable. Branch distributions gives a meaningful answer, and the amortized approach is comparable or better. 
At large scale (18M ratings) the amortized approach is thousands of nats better on test-likelihoods even after branch distributions were trained for almost ten times as long as the amortized approach took to converge (\cref{fig: real problem results}). 

%% file: method.tex

\section{Hierarchical Branched Distributions}
\label{sec: branch distributions}

\begin{figure*}[ht]
    \begin{subfigure}[b]{0.25\textwidth}
        \centering
        \resizebox*{0.65\textwidth}{!}{
            \input{./model_hbd.tex}
        }
    \end{subfigure}
    \hfill
    \begin{subfigure}[b]{0.75\textwidth}
        \centering
        \resizebox*{!}{0.35\textwidth}{
            \input{./model_hbd_expanded.tex}
        }
    \end{subfigure}
    \caption{
        The graphical model for the HBDs.
        On the left, we have plate notation for the generic HBD from \cref{eq: branch distributions alternate}. 
        Note, we can 
        have an edge from $\theta$ to $y_{ij}$ (we skip it for clarity.) 
        On the right, we have an example model with $N=3$. 
    }
    \label{fig: graph model hbd}
\end{figure*}
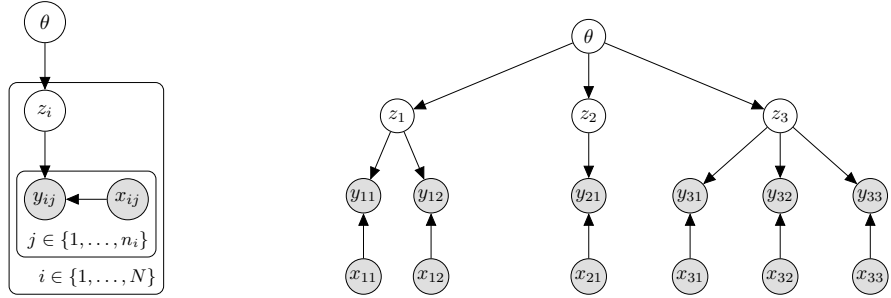
    
We focus on two-level hierarchical distributions. A generic model of this type is given by
\begin{align}
    p(\theta, z, y\vert x) & = 
            p(\theta)\prod_{i=1}^{N}p (z_i \vert \theta) p (y_{i} \vert \theta, z_i, x_{i}), \label{eq: branch distributions}
\end{align} 
where $\theta$ and $z=\{z_i\}_{i=1}^{N}$ are latent variables, 
$y=\{y_i\}_{i=1}^{N}$ are observations,
and $x=\{x_i\}_{i=1}^{N}\allowbreak$ are covariates. As the visual representations of these models 
resemble branches, we refer them as \emph{hierarchical branch distributions (HBDs)}.

\paragraph{Symmetric.} 
\phantomsection \label{sec: symmetric HBD}
We call an HBD symmetric if the conditionals are symmetric, i.e., if 
$z_i = z_j, x_i = x_j,$ and $y_i = y_j$, it implies that 
\begin{align}
    p(z_i \vert \theta) & = p(z_j \vert \theta) \text{, and }\nonumber\\
    p(y_i \vert \theta, z_i, x_i) & = p(y_j \vert \theta, z_j, x_j). \label{eq: symmetry condition}
\end{align}

\paragraph{Locally i.i.d.} Often local observations $y_i$ (and $x_i$) are a collection of conditionally i.i.d observations. Then, 
an HBD takes the form of
\begin{align}
    p(\theta, z, y\vert x) 
    & = 
    p(\theta)\prod_{i=1}^{N}p (z_i \vert \theta) \prod_{j=1}^{n_i}p (y_{ij} \vert \theta, z_i, x_{ij}), \label{eq: branch distributions alternate}
\end{align} 
where 
$y_i=\{y_{ij}\}_{j=1}^{n_i}$ and $x_i = \{x_{ij}\}_{j=1}^{n_i}$ are collections of conditionally i.i.d observations and covariates; $n_i \ge 1$ is the number of observations for branch $i$.

\paragraph{No local covariates.} Some applications do not involve the covariates $x_i$. In such cases, HBDs have a simplified form of
\begin{align}
    p(\theta, z, y) 
    & = 
    p(\theta)\prod_{i=1}^{N}p (z_i \vert \theta) p (y_{i} \vert \theta, z_i). 
\end{align}
In this paper, we will be using \cref{eq: branch distributions} and \cref{eq: branch distributions alternate} to refer HBDs---the results extend easily to case where there are no local covariates. 
(For instance, in \cref{sec: amortized dist}, we amortize using $(x_i, y_i)$  as inputs. When there are no covariates, 
we can amortize with just $y_i$.)

\subsection{Related Work}

Bayesian inference in hierarchical models is an old problem. 
The most common solutions are Markov chain Monte Carlo (MCMC)
and VI. A key advantage of VI is that gradients can sometimes
be estimated using only a subsample of data. 
\citet{hoffman2013stochastic} observe that inference in \emph{hierarchical}
models is still slow at large scale, since the number of parameters
scales with the dataset. Instead, they assume that $\theta$ and
$z_i$ from \cref{eq: branch distributions} are in conjugate exponential
families, and observe that for a
mean-field variational distribution $q(\theta)\prod_{i} q(z_i)$, 
the optimal $q(z_i)$ can be calculated in closed form for 
fixed $q(\theta)$. This is highly scalable, 
though it is limited to factorized approximations and requires a
conditionally conjugate target model.

A structured variational approximation like 
$q(\theta) \prod_{i} q(z_i \vert \theta)$ can be used which reflects
 the dependence of $z_i$ on $\theta$ 
 \cite{hoffman2015ssvi, sheth2016monte, ambrogioni2021automatic, johnson2016composing}.
However, this still has scalability problems in
general since the number of parameters grows in the size of
the data (\cref{sec: experiments}). To the best of our knowledge, 
the only approach that avoids this is the framework of structured 
stochastic VI \cite{hoffman2015ssvi, johnson2016composing}, which assumes the target is
conditionally conjugate, and that for a fixed $\theta$ an optimal
"local" distribution $q(z_i \vert \theta)$ can be calculated 
from local data. \citet{hoffman2015ssvi} address matrix factorization models and 
latent Dirichlet allocation, using Gibbs sampling to compute 
the local distributions. \citet{johnson2016composing} use amortization for conjugate models but do not consider the setting where local observations are a collection of i.i.d observations. 
Our approach is not strictly an
instance of either of these frameworks, as we do not assume conjugacy or
that amortization can exactly recover optimal 
local distributions \cite[Eq. 7]{hoffman2015ssvi}. 
Still the spirit is the same, and our approach should be seen
as part of this line of research.

Amortized variational approximations have been used to learn models with local variables \cite{kingma2013auto,edwards2016towards,ilse2019diva,bouchacourt2018multi}.
A particularly related instance of model learning is the Neural Statistician approach \cite{edwards2016towards}, where a ``statistics network'' learns representations of closely related datasets--the construction of this network is similar to our ``pooling network'' in \cref{sec: locally iid}.
However, in the Neural Statistician model there are no global variables $\theta$, and it is not obvious how to generalize their approach for HBDs. In contrast, our ``pooling network'' results from the analysis in \cref{sec: amortized dist}, reducing the architectural space when $\theta$ is present while retaining accuracy. 
Moreover, \emph{learning} a model is strictly different from our ``black-box" setting where we want to approximate the posterior of a given model,  
and the inference only has access to $\log p$ or $\nabla_{\theta, z}\log p$ (or their parts) \cite{ranganath14,kucukelbir2017automatic,ambrogioni2021automatic,aagrawal2020}.

\section{Joint approximations for HBD}

When dealing with HBDs, 
a non-structured distribution does not scale.
To see this, consider the naive VI objective---Evidence Lower Bound (ELBO, $\mathcal L$). 
Let $\qj$ be a joint variational approximation over $\rtheta$ and $\rz$; we use sans-serif font for random variables. Then,  
\begin{align}
    \ELBO{\qj}{p} = \mathbb{E}_{\qj(\rtheta, \rz)} \left[ \log \frac{p(\rtheta, \rz, y \vert x)}{\qj(\rtheta, \rz)} \right] , \label{eq: joitn elbo}
\end{align}
where $\phi$ are variational parameters.
Usually the above expectation is not tractable and one uses a Monte-Carlo estimator.
A single sample estimator is given by 
\begin{align}
    \widehat{\mathcal {L}} =  \log \frac{p(\rtheta, \rz, y \vert x)}{\qj(\rtheta, \rz)} , \label{eq: joint elbo estimate}
\end{align}
where $(\rtheta, \rz) \sim \qj$ and $ \widehat{\mathcal {L}}$ is unbiased. 
Since $\qj$ need not factorize, even a single estimate requires sampling all the latent variables at each step.
This is problematic when there are a large number of local latent variables.
One encounters the same scaling problem when taking the gradient of the ELBO.
We can estimate the gradient using any of the several available estimators \cite{ranganath14, kingma2013auto,pmlr-v32-rezende14, roeder2017sticking}; 
however, none of them scale if $\qj$ does not factorize \cite{hoffman2013stochastic}. 

\section{Branch approximations for HBD}
It is easy to see that the posterior distribution of the HBD \cref{eq: branch distributions} takes the form
$$
    p(\theta, z \vert y, x)  = 
            p(\theta \vert y, x)\prod_{i=1}^{N}p (z_i \vert \theta, y_i, x_i) \label{eq: branch posterior}.
$$
As such, it is natural to consider variational distributions that factorize the same way (see \cref{fig:q visual}.) 
In this section, we confirm this intuition---we start with any joint variational family which can have any dependence between $\theta$ and $z_1, \cdots, z_N$. 
Then, we define a corresponding ``branch'' family where $z_1, \cdots ,z_N$ are conditionally independent given $\theta$. We show that inference using the branch family will be at least as accurate as using the joint family.
We formalize the idea of branch distribution in the next definition.
\begin{defn}
    Let $\qj$ be any variational family with parameters $\phi$. We define $\qb$ to be a corresponding branch family
    if, for all $\phi$, there exists $(v, \{w_i\}_{i=1}^{N})$, such that, 
    \begin{align}
        \qb(\theta, z) = q_v(\theta) \prod_{i=1}^{N} q_{w_i} (z_i \vert \theta) = \qj (\theta) \prod_{i=1}^{N} \qj (z_i \vert \theta). \label{eq: q branch definition}
    \end{align} 
    \label{def: branch dists}
\end{defn}
\begin{wrapfigure}{r}{0.25\textwidth}
    \vspace{0pt}
    \centering
    \resizebox*{!}{0.09\textwidth}{
        \input{./model_q.tex}
    }
    \caption{$\ qb$ as in \cref{eq: q branch definition}.}
    \vspace{-12pt}
    \label{fig:q visual}
\end{wrapfigure}
Given a joint distribution, a branch distribution can always be defined by choosing $w$ and $v_i$ as the components of $\phi$ that influence $\qj(\theta)$ and  $\qj(z_i \vert \theta)$,
respectively, and choosing $q_w (\theta)$ and $q_{w_i}(z_i \vert \theta)$ correspondingly.
However, the choice is not unique (for instance, the parameterization can require transformations---different transformations can create different variants.) 

The idea to use $\qb$ is natural \cite{hoffman2015ssvi,ambrogioni2021automatic}. However, one might question if the branch variational family is as good as the original $\qj$; 
in \cref{thm: branch q better than q for hbd}, we establish that this is indeed true.

\begin{thm}
\label{thm: branch q better than q for hbd}Let $p$ be a HBD, and
 $\qj\pp{\theta,z}$ be a joint approximation family parameterized by $\phi$.
Choose a corresponding branch variational family $\qb(\theta, z)$ as in \cref{def: branch dists}.  
Then,
\[
\min_{v,w}\KL{\qb}p\leq\min_{\phi}\KL{\qj}p.
\]
\end{thm}
We stress that $\qb$ is a new variational family derived from but not identical to $\qj$.
\Cref{thm: branch q better than q for hbd} implies that we can optimize a branched variational family $\qb$ without compromising the quality of 
approximation (see \cref{app sec: proof of theorem} for proof.)
In the following corollary, we apply \cref{thm: branch q better than q for hbd} to a joint Gaussian to show that using a branch Gaussian will be equally accurate.
\begin{cor}
    \label{cor: branch gaussian}
    Let $p$ be a HBD, and let 
    \(
        \qj\pp{\theta,z}=\N\pars{\pp{\theta,z}\vert\mu,\Sigma}
    \)
    be a joint Gaussian approximation (with $\phi=\pp{\mu,\Sigma}$). Choose a variational family
    \[
    \qb\pp{\theta,z}=\N\pp{\theta\vert\mu_{0},\Sigma_0}\prod_{i=1}^N \N\pp{z_{i}\vert\mu_{i}+A_{i}\theta,\ \Sigma_i}
    \]
    with $v=\pp{\mu_{0},\Sigma_{0}}$ and $w_{i}=\pp{\mu_{i},\Sigma_{i},A_{i}}.$ 
    Then, $\displaystyle \min_{v,w}\KL{\qb}p\leq\min_{\phi}\KL{\qj}p.$ 
\end{cor}
In the above corollary, the structured family  $\qb$ is chosen such that it can represent any branched Gaussian distribution. 
Notice, the mean of the conditional distribution is an affine function of $\theta$. 
This affine relationship appears naturally when you factorize the joint Gaussian over $\left(\rtheta,\rz\right)$. 
For more details see \cref{app sec: derivation of branch gaussian}. 

\begin{figure*}[!h]
    \centering
    \hspace{-40pt}
    \hspace{-10pt}
    \subcaptionbox{Estimation with $\qj$ as in \cref{eq: joint elbo estimate}\label{fig: pseudocode joint}}
    [0.5\linewidth]{
        \resizebox*{0.175\textwidth}{!}{
            \hspace{-50pt}
            \begin{minipage}[t][0.105\textheight]{0.3\textwidth}
                \begin{algorithmic}
                \State \texttt{JointELBO}$(\phi, y, x)$
                \State \hspace{5mm}$\theta, z \sim \qj(\rtheta, \rz)$
                \State \hspace{5mm}{${\displaystyle \widehat{\mathcal L} \gets \log \frac{p(\theta, z, y \vert x)}{\qj (\theta, z)} }$}
            \end{algorithmic}
        \end{minipage}
    }}
    \subcaptionbox{
        Estimation with $\qb$ as in \cref{eq: branch elbo estimate}
        \label{fig: pseudocode branch}}
        [0.5\linewidth]{
            \resizebox*{0.425\textwidth}{!}{
        \begin{minipage}[t][0.105\textheight]{0.45\textwidth}
            \begin{algorithmic}
                \State \texttt{BranchELBO}$(v, w, y, x)$
                \State \hspace{5mm}$\theta \sim q_{v}(\rtheta)$
                \State \hspace{5mm}$z_i \sim q_{w_i}(\rz_i \vert \theta)$ for $i \in \{1, \cdots, N\}$.
                \State \hspace{5mm}{${\displaystyle \widehat{\mathcal{L}} \gets \log\frac{p\left(\theta\right)}{q_v \left(\theta\right)}+\sum_{i = 1}^{N}\log\frac{p\left(z_{i},y_{i}\vert\theta,x_{i}\right)}{q_{w_i}\left(z_{i}\vert\theta\right)} }$}
            \end{algorithmic}
        \end{minipage}
    }}
    \vspace{10pt}
    \\
    \hspace{-30pt}
    \subcaptionbox{Estimation with $\qb$ as in \cref{eq: branch susbsamples elbo estimate}
        \label{fig: pseudocode branch subsampled}}
        [0.48\linewidth]{
        \resizebox*{0.45\textwidth}{!}{
        \begin{minipage}[t][0.125\textheight]{0.49\textwidth}
            \begin{algorithmic}
                \State \texttt{SubSampledBranchELBO}$(v, w, y, x)$
                \State \hspace{5mm}$\theta \sim q_{v}(\rtheta)$
                \State \hspace{5mm}$B \sim \texttt{Minibatch($\mathsf{B}$)}$
                \State \hspace{5mm}$z_i \sim q_{w_i}(\rz_i \vert \theta)$ for $i \in B$
                \State \hspace{5mm}{${\displaystyle \widehat{\mathcal{L}} \gets \log\frac{p\left(\theta\right)}{q_v \left(\theta\right)}+\frac{N}{\verts{B}}\sum_{i\in B}\log\frac{p\left(z_{i},y_{i}\vert\theta,x_{i}\right)}{q_{w_i}\left(z_{i}\vert\theta\right)} }$}
            \end{algorithmic}
        \end{minipage}    
    }
    }
    \hspace{0pt}
    \subcaptionbox{\label{fig: pseudocode amort}Estimation with $\qab$ for $p$ as in \cref{eq: symmetry condition}}
    [0.48\linewidth]{
        \resizebox*{0.45\textwidth}{!}{
        \begin{minipage}[t][0.125\textheight]{0.49\textwidth}
            \begin{algorithmic}
                \State \texttt{AmortizedSubSampledBranchELBO}$(v, u, y, x)$
                \State \hspace{5mm}$\theta \sim q_{v}(\rtheta)$
                \State \hspace{5mm}$B \sim \texttt{Minibatch($\mathsf{B}$)}$
                \State \hspace{5mm}$w_i \gets \net_u (x_i, y_i )$ for $i \in B$
           \State \hspace{5mm}$z_i \sim q_{w_i}(\rz_i \vert \theta)$ for $i \in B$
                \State \hspace{5mm}{${\displaystyle \widehat{\mathcal{L}} \gets \log\frac{p\left(\theta\right)}{q_v \left(\theta\right)}+\frac{N}{\verts{B}}\sum_{i\in B}\log\frac{p\left(z_{i},y_{i}\vert\theta,x_{i}\right)}{q_{w_i}\left(z_{i}\vert\theta\right)} }$}
            \end{algorithmic}
        \end{minipage}
    }}

    \caption{Pseudo codes for ELBO estimation with different variational methods; 
            $w=\{w_i\}_{i=1}^N$, $y = \{y_i\}_{i=1}^{N}$, and $y_i = \{y_{ij}\}_{j=1}^{n_i}$ ($x$ is defined similar to $y$.) 
            (a) Estimates ELBO for a joint approximation; 
            (b) to (d) estimate ELBO for branch approximations;
            (c, d) use subsampling to estimate ELBO;
            (d) uses amortized conditionals; 
            (a) to (c) work for any HBD, and (d) assumes $p$ is a symmetric HBD as in \cref{eq: symmetry condition}.
            For models where $n_i > 1$, we use the $\net_u$ as in \cref{fig: featnet}.
            \texttt{Minibatch} is some distribution over the set of possible minibatches
            and $\verts{B}$ denotes the number of samples in a minibatch $B$.
            \label{fig: pseudocodes}
            }
\end{figure*}

\subsection{Subsampling in branch distributions}
In this section, we show that if $p$ is an HBD
and $\qb$  is as in \cref{def: branch dists}, 
we can estimate ELBO using local observations and scale better.  
Consider the ELBO
\begin{align}
\ELBO{\qb}{p} & =\E_{\qb\left(\rtheta,\rz\right)}\left[\log\frac{p\left(\rtheta,\rz,y\vert x\right)}{\qb\left(\rtheta,\rz\right)}\right]
\nonumber\\
 & =\E_{q_{v}\left(\rtheta\right)}\left[\log\frac{p\left(\rtheta\right)}{q_{v}\left(\rtheta\right)}\right]+\sum_{i=1}^{N}\E_{q_{v}\left(\rtheta\right)}\E_{q_{w_{i}}\left(\rz_{i}\vert\theta\right)}\left[\log\frac{p\left(\rz_{i},y_{i}\vert\theta,x_{i}\right)}{q_{w_{i}}\left(\rz_{i}\vert\theta\right)}\right]. \label{eq: branch elbo decomposition}
\end{align}
Without assuming special structure (e.g. conjugacy) the above expectations will not be available in closed form. To estimate the ELBO, let
$(\rtheta,\left\{ \rz_{i}\right\} _{i=1}^{N})\sim \qb$.
Then, an unbiased estimator is
\begin{align}
    \widehat{\mathcal{L}} & =\log\frac{p\left(\rtheta\right)}{q_{v}\left(\rtheta\right)}+\sum_{i=1}^N\left[\log\frac{p\left(\rz_{i},y_{i}\vert\rtheta,x_{i}\right)}{q_{w_{i}}\left(\rz_{i}\vert\rtheta\right)}\right]. \label{eq: branch elbo estimate}
\end{align}
Unlike the joint estimator of \cref{eq: joint elbo estimate}, one can subsample the terms in \cref{eq: branch elbo estimate} to create a new unbiased estimator. Let\textbf{ $\mathsf B$} be randomly selected minibatch of indices from $\{1,2,\dots,N\}$.
Then, 

\begin{align}
    \widehat{\mathcal{L}} & =\log\frac{p\left(\rtheta\right)}{q_{v}\left(\rtheta\right)}+\frac{N}{\verts{\mathsf{B}}}\sum_{i\in \mathsf B}\left[\log\frac{p\left(\rz_{i},y_{i}\vert\rtheta,x_{i}\right)}{q_{w_{i}}\left(\rz_{i}\vert\rtheta\right)}\right], \label{eq: branch susbsamples elbo estimate}
\end{align}
is another unbiased estimator of ELBO. In \cref{sub@fig: pseudocode branch,,sub@fig: pseudocode branch subsampled}, 
we present the complete pseudocodes for ELBO estimation with 
and without subsampling in branch distributions.  
Unsurprisingly, the same summation structure appears for gradients estimators of branch 
ELBO, allowing for efficient gradient estimation. With 
subsampled evaluation and training, 
branch distributions are immensely computationally efficient---in 
our experiments, we scale to models with $10^3$ \emph{times} more latent variables by switching to branch approximations (see \cref{fig: real problem results}).

While branch distributions are immensely more scalable than joint approximations, 
the number of parameters still scales as $\mathcal O (N)$. 
In the next section, we demonstrate that for symmetric HBDs, 
we can share parameters for the local conditionals (amortize) to allow further scalability.

\section{Amortized branch approximations}
\label{sec: amortized dist}
In this section, we discuss how one can amortize the 
local conditionals of a branch approximations when the 
target HBD is symmetric (see \cref{eq: symmetry condition}.)
We first formally introduce the amortized branch distributions in the next definition 
and then justify the amortization for symmetric HBD.
\begin{defn}
    Let $\qj$ be a joint approximation
    and let $\qb$ be as in \cref{def: branch dists}.
    Suppose $\net_u(x_i, y_i)$ is some parameterized map (with parameters $u$) 
    from local observations $(x_i, y_i)$ to space of
    $w_i$.
    Then, 
    \begin{align}
        \qab (\theta, z) = q_v (\theta) \prod_{i=1}^{N} q_{{\normalfont  \net}_u(x_i, y_i)}(z_i \vert \theta)
    \end{align}
    is a corresponding amortized branch distribution. 
    \label{def: amort branch dist}
\end{defn}
The idea to amortize is natural once you examine the 
optimization for symmetric HBDs. 
Consider the optimization for objective in \cref{eq: branch elbo decomposition}. 
\begin{align*}
\max_{v,w}\ELBO{\qb}p &=\max_{v,w} \bracs{\E_{q_{v}\left(\rtheta\right)}\left[\log\frac{p\left(\rtheta\right)}{q_{v}\left(\rtheta\right)}\right]+\sum_{i=1}^{N}\E_{q_{v}\left(\rtheta\right)}\E_{q_{w_{i}}\left(\rz_{i}\vert\rtheta\right)}\left[\log\frac{p\left(\rz_{i},y_{i}\vert\rtheta,x_{i}\right)}{q_{w_{i}}\left(\rz_{i}\vert\rtheta\right)}\right]}\\
& =\max_v \bracs{\E_{q_{v}\left(\rtheta\right)}\left[\log\frac{p\left(\rtheta\right)}{q_{v}\left(\rtheta\right)}\right]+\sum_{i=1}^{N}\max_{w_i}\E_{q_{v}\left(\rtheta\right)}\E_{q_{w_{i}}\left(\rz_{i}\vert\rtheta\right)}\left[\log\frac{p\left(\rz_{i},y_{i}\vert\rtheta,x_{i}\right)}{q_{w_{i}}\left(\rz_{i}\vert\rtheta\right)}\right]}.
\end{align*}
The crucial observation in the above equation is that for any given $v$, the optimal solution of inner optimization depends only on local
data points $\left(x_{i},y_{i}\right)$, i.e., 
\begin{align}
    w^*_i &= \argmax_{w_i}\E_{q_{v}\left(\rtheta\right)}\E_{q_{w_{i}}\left(\rz_{i}\vert\rtheta\right)}\left[\log\frac{p\left(\rz_{i},y_{i}\vert\rtheta,x_{i}\right)}{q_{w_{i}}\left(\rz_{i}\vert\rtheta\right)}\right ].
\label{eq: optimal local wi}
    \end{align}
Now, notice that if $p$ and $\qb$ have symmetric conditionals, 
then, for each $i$, 
we solve the same optimization over $w_{i}$, just with different parameters $y_i$ and $x_i$. 
Thus, one could replace the optimization over $w_i$ with an 
optimization over a parameterized function 
from $(x_i, y_i)$ to the space of $w_i$. Formally, when the network $\net_u$ 
is sufficiently capable, we make the following claim. 
\begin{claim}
    Let $p$ be a symmetric HBD and let $\qj$ be some joint 
    approximation. Let $\qab$ be as in \cref{def: branch dists}.
    Suppose that for all $v$, there exists a $u$, such that, 
    \begin{align}
        {\normalfont \net}_u (x_i, y_i) = \argmax_{w_i}\E_{q_{v}\left(\rtheta\right)}\E_{q_{w_{i}}\left(\rz_{i}\vert\rtheta\right)}\left[\log\frac{p\left(\rz_{i},y_{i}\vert\rtheta,x_{i}\right)}{q_{w_{i}}\left(\rz_{i}\vert\rtheta\right)}\right ]. 
    \end{align}
    Then, 
    \begin{align}
        \min_{v, u} \KL{\qab}p \le \min_{\phi} \KL{\qj}p
    \end{align}
    \label{claim: amortization}
\end{claim}

Note, we only amortize the conditional distribution
$q_{w_{i}}\left(z_{i}\vert\theta\right)$ and leave $q_{v}\left(\theta\right)$
unchanged. 
In practice, of course, we do not have perfect amortization functions. 
The quality of the amortization
depends on our ability to parameterize
and optimize a powerful neural network. 
In other words, we make the following 
approximation 
\begin{align}
    {\normalfont \net}_u (x_i, y_i) \approx \argmax_{w_i}\E_{q_{v}\left(\rtheta\right)}\E_{q_{w_{i}}\left(\rz_{i}\vert\rtheta\right)}\left[\log\frac{p\left(\rz_{i},y_{i}\vert\rtheta,x_{i}\right)}{q_{w_{i}}\left(\rz_{i}\vert\rtheta\right)}\right ]. 
\end{align}
In our experiments, we found the  
amortized approaches work well even with
moderately sized networks.
Due to parameter sharing, the amortized approaches converge much faster
than other alternatives (especially, true for larger models; see \cref{fig: real problem results} and \cref{tab: movielens results}). 

\section{Amortized branch approximations for i.i.d. observations}
\label{sec: locally iid}

In the previous section, we discussed how we could amortize branch approximations for symmetric HBDs. 
However, in some applications, the construction of amortization network $\net_u$ is not as straightforward. Consider the case when we have a varying number of local i.i.d observations for each local latent variable. 
In this section, we highlight the problem with naive amortization for locally i.i.d HBDs, and present a simple solution to alleviate them. 

Mathematically, for locally i.i.d HBDs we have $y_{i} = \{y_{ij}\}_{j=1}^{n_i}$ and $x_i = \{x_{ij}\}_{j=1}^{n_i}$, such that the conditional over $y_i$ factorizes as 
\[
p\left(y_{i}\vert x_{i},z_{i},\theta\right)=\prod_{j=1}^{n_{i}}p\left(y_{ij}\vert x_{ij},z_{i},\theta\right).
\]
Now, if $x_{i}$ and $y_{i}$ are directly input to the amortization network $\net_u$, the input
size to the network would change for different $i$ (notice we have
$n_{i}$ observations for $i^{\text{th}}$ local variable.) 
Another problem is that the optimal variational parameters are invariant to the order in which
the i.i.d. observations are presented. 
For instance, consider two data points: $(x_i, y_i) = [(x_{i1}, y_{i1}), \dots, (x_{i n_i}, y_{i n_i})]$ and $(x'_i, y'_i) = [(x_{i n_i}, y_{i n_i}), \dots, (x_{i1}, y_{i1})]$. 
A naive amortization scheme will evaluate very different conditionals for these two data points
because $\net_u (x_i, y_i)$ and $\net_u (x'_i, y'_i)$ will be different.

\begin{wrapfigure}{r}{0.33\textwidth}
    \vspace{-0pt}
    \begin{algorithmic}
        \State $\net_u(x_i, y_i)$
        \State \hspace{5mm}\textbf{for} {$j$\textbf{ in }$\{1, 2, \dots, n_i\}$} 
        \State \hspace{5mm}\hspace{5mm}\resizebox{0.7\hsize}{!}{$e_j \gets \texttt{feat\_net}_u (x_{ij}, y_{ij})$}
        \State \hspace{5mm}$e \gets \texttt{pool}(\{e_j\}_{j=1}^{n_i})$
        \State \hspace{5mm}$w_i \gets \texttt{param\_net}_u (e)$
        \State \hspace{5mm}\textbf{return} $w_i$
    \end{algorithmic}
    \caption{Psuedocode for $\net_u$ for locally i.i.d symmetric HBD.
    \label{fig: featnet}}
    \vspace{-10pt}
\end{wrapfigure}
To deal with both issues: variable length input and permutation
invariance, we suggest learning a ``feature network'' and "pooling function" based amortization network; 
this is reminiscent of "deep sets" \cite{deepsets} albeit here intended not just to enforce permutation invariance but also to deal with inputs of different sizes.  
Firstly, a feature network \texttt{feat\_net} takes each $\pp{x_{ij},y_{ij}}$ pair and returns a vector of features $e_j$.
Secondly, a pooling function \texttt{pool} takes the collection $\{e_j\}_{j=1}^{n_i}$ and achieves the two aims. 
First, it collapses $n_i$ feature vectors into a single fixed-sized feature $e$ (with the same dimensions as $e_j$). 
Second, pooling is invariant by construction to the order of observations (for example, pooling function would take a dimension-wise mean or sum across $j$.) 
Finally, this pooled feature vector $e$ is input to another network \texttt{param\_net} that
returns the final parameters $w_{i}$.
The pseudocode for a $\net_u$ with feature networks is available in \cref{fig: featnet}.
In \cref{tab: method applicability}, in appendix, we summarize the applicability of proposed variational methods to different HBD variants.

\begin{table}[t]
    \caption{All variational families used in our experiments. $\Sigma$ denotes a 
    generic covariance matrix and $\sigma^2$ denotes a diagonal 
    covariance.}
    \label{tab: methods summary}
    \begin{center}
    \resizebox*{\textwidth}{!}{
            \begin{tabular}[!h]{@{} llll @{}}
                \toprule
                \textbf{Gaussian Family} & $\qj$ as in \cref{eq: joitn elbo} & $\qb$ as in \cref{eq: q branch definition} & $\qab$ as in \cref{def: amort branch dist}\\
                \midrule
                Dense & $\N (\theta, z \vert \mu, \Sigma)$ & \resizebox{0.35\hsize}{!}{$\mathsmaller{\displaystyle \N (\theta \vert \mu_0, \Sigma_0) \prod_{i=1}^{N} \N (z_i \vert \mu_i + A_i \theta, \Sigma_i)}$} & \resizebox{0.35\hsize}{!}{${\displaystyle \N (\theta \vert \mu_0, \Sigma_0) \prod_{i=1}^{N} \N (z_i \vert \mu_i + A_i \theta, \Sigma_i)}$}\\
                & {\small $\phi = (\mu, \Sigma)$}&  {\small $v = (\mu_0, \Sigma_0)$, $w_i  = (\mu_i, A_i, \Sigma_i)$}&  {\small $v = (\mu_0, \Sigma_0), (\mu_i, A_i, \Sigma_i) = \net_{u}(x_i, y_i)$}\\[1.5mm]
                Block Diagonal & $\N (\theta \vert \mu_0, \Sigma_0)\N (z \vert \mu_1, \Sigma_1)$ & \resizebox{0.3\hsize}{!}{${\displaystyle \N (\theta \vert \mu_0, \Sigma_0) \prod_{i=1}^{N} \N (z_i \vert \mu_i, \Sigma_i)}$} & \resizebox{0.3\hsize}{!}{${\displaystyle \N (\theta \vert \mu_0, \Sigma_0) \prod_{i=1}^{N} \N (z_i \vert \mu_i, \Sigma_i)}$}\\
                &  {\small $\phi = (\mu_0, \mu_1, \Sigma_0, \Sigma_1)$}&  {\small $v = (\mu_0, \Sigma_0)$, $w_i = (\mu_i, \Sigma_i)$} &  {\small $v = (\mu_0, \Sigma_0) \text{, }(\mu_i, \Sigma_i) = \net_{u}(x_i, y_i)$}\\[1.5mm]
                Diagonal & $\N (\theta, z \vert \mu,  \sigma^2)$ & \resizebox{0.3\hsize}{!}{${\displaystyle \N (\theta \vert \mu_0, \sigma_0^2) \prod_{i=1}^{N} \N (z_i \vert \mu_i, \sigma_i^2)}$} & \resizebox{0.3\hsize}{!}{${\displaystyle
                \N (\theta \vert \mu_0, \sigma^2_0) \prod_{i=1}^{N} \N (z_i \vert \mu_i, \sigma^2_i)}$}\\[1.5mm]
                    &  {\small $\phi = (\mu, \sigma^2)$}&  {\small $v = (\mu_0, \sigma^2_0)$, $w_i = (\mu_i, \sigma_i^2)$}&  {\small $v = (\mu_0, \sigma_0^2)$, $(\mu_i, \sigma_i^2) = \net_{u}(x_i, y_i)$}\\
                \bottomrule                
            \end{tabular}
         }
        \end{center}
\end{table}

\section{Experiments}
\label{sec: experiments}
\begin{wrapfigure}{r}{0.32\textwidth}
    \vspace{-12pt}
    \centering
    \includegraphics[trim = 130 600 120 10, width=0.2\textwidth]{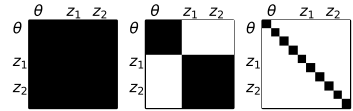}
    \caption{
         Visualization of covariances of 
        dense, block-diagonal, 
        and diagonal Gaussians.
        For each, 
        we experiment with  
        $\qj$, $\qb$, and $\qab$ methods.}
    \vspace{-22pt}
    \label{fig:cov visual}
\end{wrapfigure}
We conduct experiments on a synthetic and a real-world problem. 
For each, we consider three inference methods: 
using a joint distribution, using a branch variational approximation, 
and using our amortized approach. 

For each method, we consider three variational approximations 
a completely diagonal Gaussian, 
a block-diagonal Gaussian (with blocks for $\theta$ and $z$) and 
a dense Gaussian; see \cref{fig:cov visual} for a visual. 
(For each choice of a joint distribution,
the corresponding $\qb$ is used for the branch variational
approximation and the corresponding $\qab$ for the amortized
approach; see \cref{tab: methods summary} for details.) 

We use reparameterized gradients \cite{kingma2013auto} and  optimize with Adam \cite{kingma2014adam}
(see \cref{app sec: experimental details} for complete experimental details.) 
In \cref{app sec: practicaL challenges}, we discuss some of the 
observations we had during experimentation involving parameterization,
\texttt{feat-net} architecture, batch-size selection, initialization, and gradient estimators. 

\subsection{Synthetic problem}
The aim of the synthetic experiment is to work with models where 
we have access to closed-form posterior (and the marginal likelihood.)
If the variational family contains the posterior,
one can expect the methods to perform close to ideal (provided we can optimize well.) 

For our experiments, we use the following hierarchical regression model
\begin{align}
    p(\theta, z, y \vert x) &= \mathcal{N}(\theta \vert 0, I) \prod_{i=1}^{N} \mathcal{N}(z_i \vert \theta, I) \prod_{j=1}^{n_i}\mathcal{N}(y_{ij} \vert x_{ij}^\top z_i, 1),
\end{align}
where $\mathcal{N}$ denotes a Gaussian distribution, 
and $I$ is an identity matrix (see \cref{app sec: synthetic details} 
for posterior and the marginal closed-form expressions.)

To demonstrate the performance of our methods, we experiment with three different problems scale
(correspond to three different models with $N =$ 10, 1K, and $100$K.) 
Synthetic data is created using forward sampling for each of the scale variants independently. 
We avoid any test data and metrics for the synthetic problem as the log-marginal is known in closed form.
The inference results are present in \cref{fig: synthetic results} (in appendix.) 
In all cases, amortized distributions perform favorably when compared to branch and joint distributions. 
\begin{table}[b]
    \centering
    \caption{
        Inference results for the MovieLens25M problem. 
        For both metrics, we 
        draw a fresh batch of 10,000 samples from the 
        final posterior. All values are in nats (higher is better).
        \vspace{7pt}}
        \label{tab: movielens results}
    \resizebox*{0.9\textwidth}{!}{
        \begin{tabular}[h]{@{} ll rrr rrr @{}}
            \toprule
            \multicolumn{2}{@{}l@{}}{Metric} & \multicolumn{3}{c}{Final ELBO} & \multicolumn{3}{c}{Test likelihood}\\
            \multicolumn{2}{@{}l@{}}{$\approx$ \# train ratings} & 2.5K & 180K & 18M & 2.5K & 180K & 18M\\[2.5mm]
            \multicolumn{2}{@{}l@{}}{Methods (see \cref{tab: methods summary})} & \multicolumn{6}{l}{}\\
            \midrule
            Dense & $\qj$ &   -1572.31 &             &             & -166.37 &           &             \\
                        & $\qb$ &   -1572.39 & -1.0368e+05 & -1.1413e+07 & -166.66 & -11054.43 & -1.3046e+06 \\
                        & $\qab$ &   -1572.45 & -1.0352e+05 & -1.0665e+07 & -166.64 & -10976.38 & -1.1476e+06 \\[3mm]
            Block  & $\qj$ &   -1579.04 &             &             & -167.36 &           &             \\
            Diagonal            & $\qb$ &   -1579.05 & -1.0350e+05 & -1.1078e+07 & -166.97 & -10987.17 & -1.2538e+06 \\
                        & $\qab$ &   -1579.06 & -1.0353e+05 & -1.0665e+07 & -166.96 & -10975.96 & -1.1484e+06 \\[3mm]
            Diagonal & $\qj$ &   -1592.59 &             &             & -167.39 &           &             \\
                        & $\qb$ &   -1592.64 & -1.0428e+05 & -1.1325e+07 & -167.31 & -10977.95 & -1.2713e+06 \\
                        & $\qab$ &   -1592.64 & -1.0430e+05 & -1.0736e+07 & -167.29 & -10980.75 & -1.1497e+06 \\
               \bottomrule
        \end{tabular}
    }
\end{table}

\subsection{MovieLens}

Next, we test our method on the MovieLens25M \cite{harper2015movielens}, a dataset of 25 million movie ratings for over 62,000 movies, rated by 162,000 users, along with a set of features (tag relevance scores \cite{vig2012tag}) for each movie.

Purely, to make experiments more efficient on GPU hardware, we pre-process the data to drop users with more than 1,000 ratings---leaving around $20$M ratings.
Also, for the sake of efficiency, we PCA the movie features to reduce their dimensionality to 10. 
We used a train-test split such that, for each user, one-tenth of the ratings are in the test set. 
This gives us $\approx$ 18M ratings for training (and $\approx$ 2M ratings for testing.) 
\begin{figure*}[!h]
    \centering
    \resizebox*{\textwidth}{!}{
        \begin{subfigure}[b]{\textwidth}
            \centering
            \includegraphics[trim=0 0 0 0,clip,width=\textwidth]{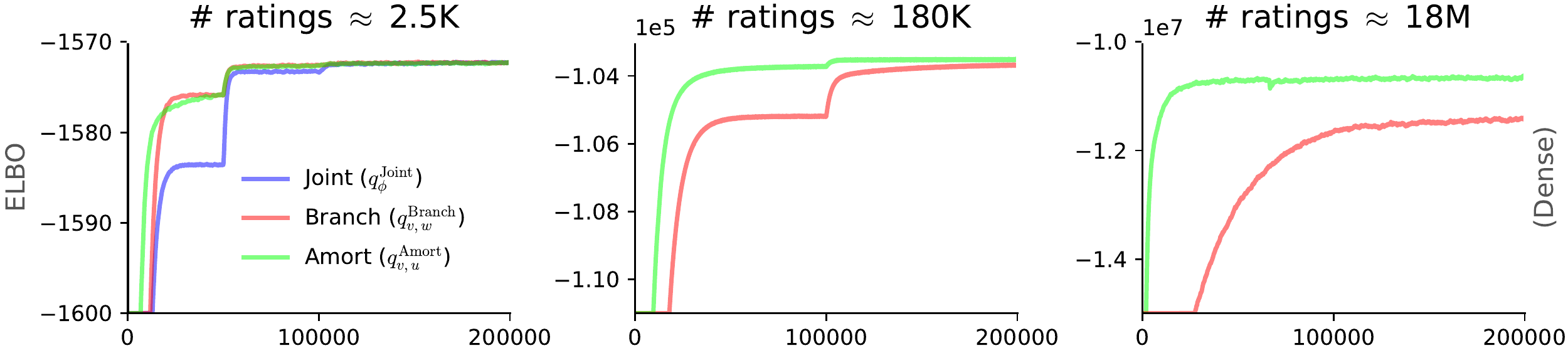}
        \end{subfigure}
    }
    \resizebox*{\textwidth}{!}{
            \begin{subfigure}[b]{\textwidth}
                \centering
                \includegraphics[trim=0 0 0 0,clip,width=\textwidth]{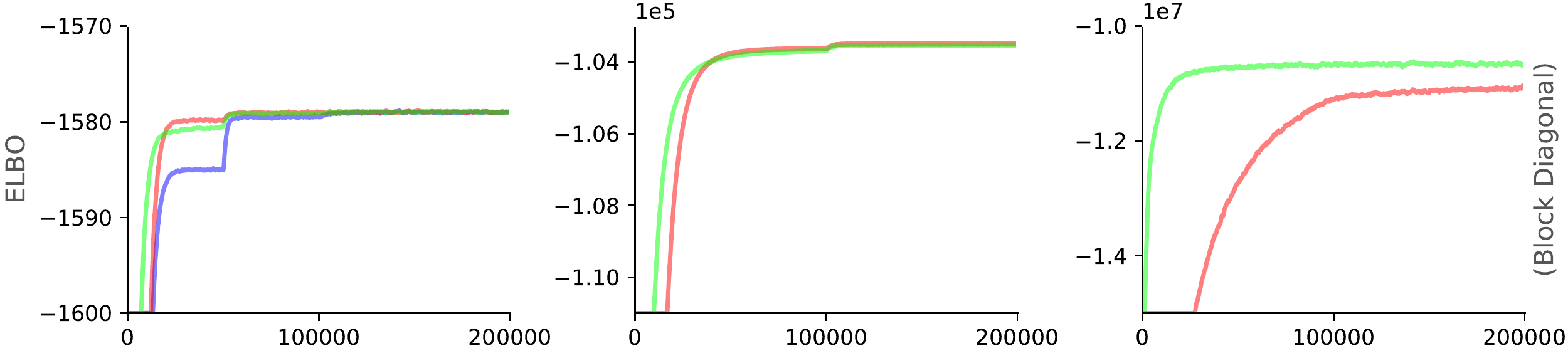}
            \end{subfigure}
    }
    \resizebox*{\textwidth}{!}{
                \begin{subfigure}[b]{\textwidth}
                    \centering
                    \includegraphics[trim=0 0 0 0,clip,width=\textwidth]{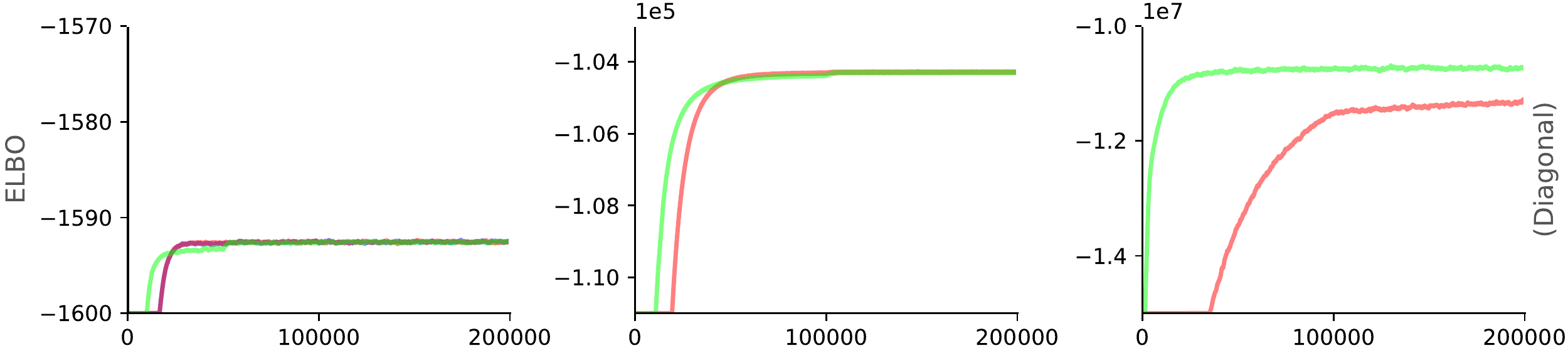}
                \end{subfigure}
    }
       \caption{
        Training ELBO trace for the MovieLens25M problem.
        Top to bottom:  dense, 
        block diagonal, and diagonal Gaussian  
        (for each, we have $\qj$, $\qb$, and $\qab$ method.)
        Left to right: small, moderate, and large scale of the MovieLens25M problem.   
        For clarity, we plot the exponential moving average
        of the training ELBO trace with a smoothing value of 0.001. 
        These traces correspond to the values reported in \cref{tab: movielens results}.}
        \label{fig: real problem results}
\end{figure*}

We use the hierarchical model
\begin{align}
    p(\theta, z, y \vert x) &= \mathcal{N}(\theta \vert 0, I) \prod_{i=1}^{N} \mathcal{N}(z_i \vert \mu(\theta), \Sigma(\theta)) \prod_{j=1}^{n_i}\mathcal{B}(y_{ij} \vert \mathrm{sigmoid} (x_{ij}^\top z_i) ),
\end{align}
where $\theta$ represents distribution over user preferences; 
for instance, $\theta$ might represent that users who like action films tend to also like thrillers but tend to dislike musicals; 
$z_i$ determine the user specific preference; 
$x_{ij}$ are the features of the $j^{\mathrm {th}}$ movie rated by user $i$; 
$y_{ij}$ is the binary movie ratings;
$n_i$ is the number of movies rated by user $i$, and
$\mathcal{B}$ denotes a Bernoulli distribution. Here, $\mu$ and $\Sigma$ are functions of $\theta$, such that, for $\theta=[\theta_\mu, \theta_\Sigma]$, we have
$$
\mu(\theta) = \theta_\mu,  \quad \quad\text{ and } \quad  \quad \Sigma(\theta) = \mathrm{tril}(\theta_\Sigma)^\top \mathrm{tril}(\theta_\Sigma)
$$
where $\mathrm{tril}$ is a function that transforms an unconstrained vector into a lower-triangular positive Cholesky factor. 
As movie features $x_{ij}\in \R^{10}$, we have $\theta_\mu \in \R^{10}, \theta_\Sigma \in \R^{55}$, and $z_i \in \R^{10}$. 
Note that as $\Sigma$ depends on $\theta$, and the likelihood is Bernoulli, the model is non-conjugate. 

For inference, we use the methods as described in \cref{tab: methods summary}. 
Note, we hold the amortization network architecture constant 
across the scales--the number of parameters remains fixed for $\qab$ (for all Gaussian variants)
while the number of parameters scale as $\mathcal{O} (N)$ for $\qb$ (see \cref{app sec: derivation of branch gaussian} for more details.)  

In \cref{fig: real problem results}, we plot the training time ELBO trace, 
and in \cref{tab: movielens results}, we present the final training ELBO and test likelihood values. 
We approximate the test likelihood $p(y^{\mathrm{test}} \vert x^{\mathrm{test}}, x, y)$ with
$\mathbb{E}_q \left[ p(y^{\mathrm{test}} \vert x^{\mathrm{test}}, q) \right]$, and draw a
fresh batch of 10,000 samples to approximate the expectation (see \cref{app sec: experimental details} for complete details.)
For the smaller model, the amortized and branch approaches perform 
similar to the joint approach for all three variational approximations; this supports \cref{thm: branch q better than q for hbd} and \cref{claim: amortization}. 
For the moderate size model, the branch and 
amortize approaches are very comparable to each other, 
while joint approaches fail to scale. 
For the large model (18M ratings), 
amortized approaches are significantly
better than branch methods. 
We conjecture this is 
because parameter sharing in amortized approaches 
improves convergence for models where batch size is 
smaller compared to total iterates--true only for large model. 
Interestingly, the performance of amortized Dense and 
Block Gaussian approximation 
is very similar in the large and moderate setting 
(see $\qab$ results in \cref{tab: movielens results,tab: extended movielens a,tab: extended movielens b}.)
We conjecture this is because the posterior over the global parameters is very concentrated for this problem. 
As $\theta$ behaves like a single fixed value, Block Gaussian performs just as good as the Dense approach (see \cref{app sec: tradeoffs} for more discussion.)

\section{Discussion}
In this paper, we present structured amortized variational 
inference scheme that can scale to large hierarchical models
without losing benefits of joint approximations.
Such models are ubiquitous in social sciences, ecology, 
epidemiology, and other fields. 
Our ideas can not only inspire further research in inference 
but also provide a formidable baseline for applications.


%% file: model_hbd.tex
\begin{tikzpicture}[x=0.7cm,y=0.8cm, ->]

  
    \node[latent]    (theta)  {$\theta$}; %
    \node[latent, below=of theta]    (z_i)  {$z_i$}; %
    \node[obs, below=of z_i]    (y_ij)  {$y_{ij}$}; %
    \node[obs, right=of y_ij]    (x_ij)  {$x_{ij}$}; %

    \edge{theta}{z_i};
    \edge{z_i}{y_ij};
    \edge{x_ij}{y_ij};

    \plate{user_i}{(x_ij)(y_ij)}{$j \in \{1, \dots, n_i\}$}
    \plate{user_all}{(z_i)(user_i)}{$i \in \{1, \dots, N\}$}













\end{tikzpicture}

%% file: model_hbd_expanded.tex

\begin{tikzpicture}[x=0.7cm,y=0.8cm, ->]

    \node[latent, minimum size=0.6cm]    (theta)  {$\theta$}; %

    \node[latent, below=of theta, xshift = -3.4cm, minimum size=0.6cm]    (z_1)  {$z_1$}; %
    \node[latent, below=of theta, minimum size=0.6cm]    (z_2)  {$z_2$}; %
    \node[latent, below=of theta, xshift = 3.4cm, minimum size=0.6cm]    (z_3)  {$z_3$}; %

    \node[obs, below=of z_1, xshift = -0.6cm, minimum size=0.6cm]    (y_11)  {${y_{11}}$}; %
    \node[obs, below=of z_1, xshift = 0.6cm, minimum size=0.6cm]    (y_12)  {${y_{12 }}$}; %

    \node[obs, below=of z_2, minimum size=0.6cm]    (y_21)  {${y_{21}}$}; %
    
    \node[obs, below=of z_3, xshift = -1.6cm, minimum size=0.6cm]    (y_31)  {${y_{31}}$}; %
    \node[obs, below=of z_3, minimum size=0.6cm]    (y_32)  {${y_{32}}$}; %
    \node[obs, below=of z_3, xshift = 1.6cm, minimum size=0.6cm]    (y_33)  {${y_{33}}$}; %

    \node[obs, below=of y_11, minimum size=0.6cm]    (x_11)  {${x_{11}}$}; %
    \node[obs, below=of y_12, minimum size=0.6cm]    (x_12)  {${x_{12}}$}; %
    
    \node[obs, below=of y_21, minimum size=0.6cm]    (x_21)  {${x_{21}}$}; %

    \node[obs, below=of y_31, minimum size=0.6cm]    (x_31)  {${x_{31}}$}; %
    \node[obs, below=of y_32, minimum size=0.6cm]    (x_32)  {${x_{32}}$}; %
    \node[obs, below=of y_33, minimum size=0.6cm]    (x_33)  {${x_{33}}$}; %

    \edge{theta}{z_1, z_2, z_3}

    

    

    \edge{z_1}{y_11, y_12}
    \edge{z_2}{y_21}
    \edge{z_3}{y_31, y_32, y_33}
    
    
    \edge{x_11}{y_11};
    \edge{x_12}{y_12};

    \edge{x_21}{y_21};

    \edge{x_31}{y_31};
    \edge{x_33}{y_33};
    \edge{x_32}{y_32};

\end{tikzpicture}

%% file: model_q.tex

\begin{tikzpicture}[x=0.7cm,y=0.5cm, ->]

    \node[latent, minimum size=0.7cm]    (theta)  {$\theta$}; %

    \node[latent, below=of theta, xshift = -2cm, minimum size=0.7cm]    (z_1)  {$z_1$}; %
    \node[latent, below=of theta, minimum size=0.7cm, xshift = -1cm]    (z_2)  {$z_2$}; %
    \node[latent, below=of theta, xshift =2cm, minimum size=0.7cm]    (z_N)  {$z_N$}; %

    \edge{theta}{z_1, z_2, z_N}

    \path (z_2) -- node[auto=false]{\ldots} (z_N);

\end{tikzpicture}

%% file: appendix.tex

\appendix

\section{Other practical challenges}
\label{app sec: practicaL challenges}
\paragraph{Parameterizing covariances} 
It is common to to re-parameterize covariance matrices to a vector of unconstrained parameters. 
As above, the typical way to do this is via a function $\mathrm{tril}$ that maps unconstrained 
vectors to Cholesky factors, i.e. lower-triangular matrices with positive diagonals. 
This can be done by simple re-arranging the components of the vector 
into a lower-triangular matrix, followed by applying a function to map
the entries on the diagonal components to the positive numbers.
In our preliminary experiments, the choice of the mapping was 
quite significant in terms of how difficult optimization was. Common choices
like the ${\exp(x)}$ and ${\log (\exp (x) + 1)}$ functions
did not perform well when the outputs were close to zero.
Instead, we propose to use the transformation 
$ x \mapsto \frac{1}{2}(x + \sqrt{x^2 + 4 \gamma})$, 
where $\gamma$ is a hyperparameter (we use $\gamma=1$). 
This is based on the proximal operator for the multivariate Gaussian entropy 
\cite[section 5]{domke2020provable}. 
Intuitively, when $x$ is a large positive number, 
this mapping returns approximately $x$, while if $x$ is 
a large negative number, the mapping returns approximately $-1/x$. 
This decays to zero more slowly than common mappings, 
which appears to improve numerical stability and 
the conditioning of the optimization.

\paragraph{Feature network architecture.}
In this paper, we propose the use of a separate \texttt{feat\_net} to deal with variable-length input and order invariance.
In our preliminary experiments, we found that the performance improves when we concatenate the embedding $e_j$ in the \cref{fig: featnet} with it's dimension-wise square before sending it to the pooling function \texttt{pool}.
We hypothesis that this is because the embeddings act as learnable statistics, and using the elementwise square directly provides useful information to \texttt{param\_net}. 

\paragraph{Batch size selection}
For the small scale problems (both synthetic and MovieLens), 
we do not subsample data; this is to maintain a fair comparison to 
joint approaches that do not support subsampling.
For moderate and large scales, we select the batch size for the branch methods 
based on the following rule of thumb: we increase the batch size such 
that $\mathsmaller {\frac{\verts{B}}{T^{1.18}}}$ is roughly maximized. 
This rule-of-thumb captures the following intuition. 
Suppose batch size $\verts{B}$ takes time $T$ per iteration.  
If the time taken per iteration for batchsize $\verts{2B}$ is less than $1.8T$, then we should use $2B$.    
Of course, we roughly maximize $\mathsmaller{\frac{\verts{B}}{T^{1.18}}}$ for computational ease. 
We use the same batch size as the branch methods for the amortized methods as we found that amortized methods 
were much more robust to the choice of batchsize. 
We use $\verts {B} = 200$ for moderate scale and $\verts{B}= 400$ for large scale.  


\paragraph{Initialization}
We initialize the neural network parameters using a truncated normal distribution with zero mean 
and standard deviation equal to $\mathsmaller{\sqrt{1/\mathrm{fan\_in}}}$, where $\mathrm{fan\_in}$ is the number of 
inputs to the layer \cite{lecun2012efficient}. We initialize the final output layer of the 
\texttt{param-net} with a zero mean Gaussian with a standard deviation 
of 0.001 \cite{aagrawal2020}. This ensures an almost standard normal initialization for the local conditional $q_{\net_{u}(x_i, y_i)}(z_i \vert \theta)$. 

\paragraph{Gradient Calculation}
In our preliminary experiments, we found sticking the landing gradient \cite[STL]{roeder2017sticking} to be 
less stable. STL requires a re-evaluation of the density which in turn requires a matrix inversion (for the Cholesky factor); this 
matrix inversion was sometimes prone to numerical precision errors. Instead, we found the 
regular gradient, also called the total gradient in \cite{roeder2017sticking}, to be numerically robust as it can be evaluated 
without a matrix inversion. 
This is done by simultaneously sampling and evaluating the density in much the same as done in normlaizing flows \cite{pmlr-v32-rezende14,papamakarios2019normalizing}.
We use the total gradient for all our experiments.  
\section{General trade-offs}
\label{app sec: tradeoffs}

\begin{wraptable}{r}{0.5\textwidth}
    \vspace{-10pt}
    \caption{Summary of method applicability.}
    \label{tab: method applicability}
    \resizebox*{0.5\textwidth}{!}{
        \begin{tabular}[!h]{@{} lccc @{}}
            \toprule
            Models & $\qb$ & $\qab$ & $\qab$ w/ \texttt{feat\_net}$_u$\\
            &{\small (\cref{def: branch dists})}& {\small (\cref{def: amort branch dist})}&{(\small $\net_u$ as in \cref{fig: featnet})}\\
            \midrule
            HBD \hspace{44pt} {(\small \cref{eq: branch distributions})}  & \checkmark & $\times$ & $\times$\\[1.5mm]
            Symmetric HBD {(\small \cref{eq: symmetry condition})}& \checkmark & \checkmark & \checkmark\\[1.5mm]
            Locally i.i.d.  &&&\\
            Symmetric HBD {(\small\cref{eq: branch distributions alternate})} & \checkmark & $\times$ & \checkmark\\
            \bottomrule
        \end{tabular}
        }    
\end{wraptable}

The amortized approach proposed in \cref{sec: amortized dist} is only applicable for symmetric models. 
In \cref{tab: method applicability}, we summarize the applicability of all the methods we discuss in this paper. 
In our preliminary experiments, for amortized approaches, the performance improved when we increased the number of layers in the 
neural networks or increased the length of the embeddings in $\net_u$. However, we make no serious efforts to find the 
optimal architecture. In fact, we use the same architecture for all our experiments, across the scales. We believe the 
performance on a particular task can be further improved by carefully curating the neural architecture. Note that there are 
no architecture choices in joint or branched approaches. We also did not optimize our choice of number of samples to drawn from $q$ to estimate ELBO. 
This forms the second source of stochasticity and using more samples can help reduce the variance \cite{titsias2014doubly}. 
We use 10 copies for all our experiments.   

A particularly interesting case arises when the number of local latent variables ($N$) is very large. In such scenarios, the true posterior $p(\theta \vert x, y)$
can be too concentrated. As the randomness in $\theta$ is very low, we might not gain any significant benefits from conditioning on $\theta$\textemdash as $\theta$ 
reduces to a fixed quantity, Dense Gaussian will work just well as the Block Gaussian (see \cref{tab: movielens results,tab: extended movielens a,tab: extended movielens b}). 
In practice, it is hard to know this apriori; in fact, our scalable approaches allow for such analysis on large scale model.

\section{Proof for Theorem}
\label{app sec: proof of theorem}

\begin{thm*}
    [Repeated] Let $p$ be a HBD, and
     $\qj\pp{\theta,z}$ be a joint approximation family parameterized by $\phi$.
    Choose a corresponding branch variational family $\qb(\theta, z)$ as in \cref{def: branch dists}.  
    Then,
    \[
    \min_{v,w}\KL{\qb}p\leq\min_{\phi}\KL{\qj}p.
    \]
\end{thm*}

\begin{proof}
    Construct a new distribution $q'_\phi$ such that 
    \begin{align}
        q'_\phi(\theta, z) = \qj (\theta) \prod_{i} \qj(z_i \vert \theta). 
    \end{align}
    Then, note that $z_i$ are conditionally independent in $q'_\phi$, such that,  
    \begin{align}
        q'_\phi(z_i \vert \theta, z_{<i}) & = q'_\phi(z_i \vert \theta).
    \end{align}
    From chain rule of KL-divergence, we have
    \begin{align*}
        \KL{\qj(\rtheta, \rz)}{p(\rtheta, \rz \vert x, y)} &= \KL{\qj(\rtheta)}{p(\rtheta \vert x, y)} \\
        & \phantom{some}+ \sum_i \KL {\qj(\rz_i \vert \rz_{<i}, \rtheta)}{p(\rz_i \vert \rz_{<i}, \rtheta, x, y)},\text{ and}\\
        \KL{q'_\phi(\rtheta, \rz)}{p(\rtheta, \rz \vert x, y)} &= \KL{\qj(\rtheta)}{p(\rtheta \vert x, y)} \\
        & \phantom{some}+ \sum_i \KL {q'_\phi(\rz_i \vert \rz_{<i}, \rtheta)}{p(\rz_i \vert \rz_{<i}, \rtheta, x, y)}.\\
    \end{align*} 
    Consider any arbitrary summand term. We have that 
    \begin{align*}
        &\KL {\qj(\rz_i \vert \rz_{<i}, \rtheta)}{p(\rz_i \vert \rz_{<i}, \rtheta, x, y)}\\ 
        & \overset{(1)}{=} \KL {\qj(\rz_i \vert \rz_{<i}, \rtheta)}{p(\rz_i \vert \rtheta, x_i, y_i)}\\
        & \overset{(2)}{=} \E_{\theta \sim \qj(\rtheta)}\E_{z_{< i} \sim \qj (\rz_{<i} \vert \rtheta)} [\KL{\qj(\rz_i \vert z_{<i}, \theta)}{p(\rz_i \vert \theta, x_i, y_i)}]\\
        & \overset{(3)}{\ge} \E_{\theta \sim \qj(\rtheta)} \left[\KL{\E_{z_{< i} \sim \qj (\rz_{<i} \vert \rtheta)}\left[\qj(\rz_i \vert z_{<i}, \theta)\right]}{\E_{z_{< i} \sim \qj (\rz_{<i} \vert \rtheta)}\left[p(\rz_i \vert \theta, x_i, y_i)\right]}\right]\\
        & \overset{(4)}{=} \E_{\theta \sim \qj(\rtheta)} \left[\KL{ \qj(\rz_i \vert \theta)}{p(\rz_i \vert \theta, x_i, y_i)}\right]\\
        & \overset{}{=} \KL{ \qj(\rz_i \vert \rtheta)}{p(\rz_i \vert \rtheta, x_i, y_i)}\\
        & \overset{}{=} \KL{ q'_\phi(\rz_i \vert \rtheta)}{p(\rz_i \vert \rtheta, x_i, y_i)}\\
        & \overset{(5)}{=} \KL{ q'_\phi(\rz_i \vert \rtheta, \rz_{<i})}{p(\rz_i \vert \rtheta, \rz_{<i}, x_i, y_i)},\\
    \end{align*}
    where (1) follows from HBD structure; (2) follows from definition of conditional KL divergence; (3) follows from convexity of KL divergence and Jensen's inequality; (4) follows from marginalization, and (5) follows from the conditional independence of $q'$ and $p$.
    Summing the above result over $i$ gives that 
    \begin{align}
        \KL{q'_\phi}{p} \le \KL{\qj}{p}.
    \end{align}
    Now, from \cref{def: branch dists}, we know that for every $\phi$, there exists a corresponding $(v, w)$ such that $\qb = q'_\phi.$
    Let $\phi^* = \argmin_\phi \KL {\qj}{p}$. Then, there exists some $\qb = q'_{\phi^*}$. 
    Then, it follows that 
    \begin{align}
        \min_{v,w}\KL{\qb}p\le \KL{q'_{\phi^*}}{p} \le \KL{q^{\mathrm{Joint}}_{\phi^*}}{p} = \min_{\phi}\KL{\qj}p.
    \end{align}
\end{proof}

\section{Proof for Claim}
\label{app sec: proof of claim}
\begin{claim*}[Repeated]
    Let $p$ be a symmetric HBD and let $\qj$ be some joint 
    approximation. Let $\qab$ be as in \cref{def: branch dists}.
    Suppose that for all $v$, there exists a $u$, such that, 
    \begin{align}
        {\normalfont \net}_u (x_i, y_i) = \argmax_{w_i}\E_{q_{v}\left(\rtheta\right)}\E_{q_{w_{i}}\left(\rz_{i}\vert\rtheta\right)}\left[\log\frac{p\left(\rz_{i},y_{i}\vert\rtheta,x_{i}\right)}{q_{w_{i}}\left(\rz_{i}\vert\rtheta\right)}\right ]. 
    \end{align}
    Then, 
    \begin{align}
        \min_{v, u} \KL{\qab}p \le \min_{\phi} \KL{\qj}p
    \end{align}
\end{claim*}

\begin{proof}
    Consider the optimization for $\qb$. We have
    \begin{align*}
        \max_{v,w}\ELBO{\qb}p 
        &=\max_{v,w} \bracs{\E_{q_{v}\left(\rtheta\right)}\left[\log\frac{p\left(\rtheta\right)}{q_{v}\left(\rtheta\right)}\right]+\sum_{i=1}^{N}\E_{q_{v}\left(\rtheta\right)}\E_{q_{w_{i}}\left(\rz_{i}\vert\rtheta\right)}\left[\log\frac{p\left(\rz_{i},y_{i}\vert\rtheta,x_{i}\right)}{q_{w_{i}}\left(\rz_{i}\vert\rtheta\right)}\right]}\\
        & =\max_v \bracs{\E_{q_{v}\left(\rtheta\right)}\left[\log\frac{p\left(\rtheta\right)}{q_{v}\left(\rtheta\right)}\right]+\sum_{i=1}^{N}\max_{w_i}\E_{q_{v}\left(\rtheta\right)}\E_{q_{w_{i}}\left(\rz_{i}\vert\rtheta\right)}\left[\log\frac{p\left(\rz_{i},y_{i}\vert\rtheta,x_{i}\right)}{q_{w_{i}}\left(\rz_{i}\vert\rtheta\right)}\right]}\\
        & \overset{(1)}{=}\max_v \bracs{\E_{q_{v}\left(\rtheta\right)}\left[\log\frac{p\left(\rtheta\right)}{q_{v}\left(\rtheta\right)}\right]+\sum_{i=1}^{N}\E_{q_{v}\left(\rtheta\right)}\E_{q_{\net_u(x_i, y_i)}\left(\rz_{i}\vert\rtheta\right)}\left[\log\frac{p\left(\rz_{i},y_{i}\vert\rtheta,x_{i}\right)}{q_{\net_u(x_i, y_i)}\left(\rz_{i}\vert\rtheta\right)}\right]}\\
        & {\le} \max_u \max_v \bracs{\E_{q_{v}\left(\rtheta\right)}\left[\log\frac{p\left(\rtheta\right)}{q_{v}\left(\rtheta\right)}\right]+\sum_{i=1}^{N}\E_{q_{v}\left(\rtheta\right)}\E_{q_{\net_u(x_i, y_i)}\left(\rz_{i}\vert\rtheta\right)}\left[\log\frac{p\left(\rz_{i},y_{i}\vert\rtheta,x_{i}\right)}{q_{\net_u(x_i, y_i)}\left(\rz_{i}\vert\rtheta\right)}\right]}\\
        & = \max_u \max_v \ELBO{\qab}p,
        \end{align*}
        where (1) follows from the assumption in the Claim.
        Now, from the ELBO decomposition equation, we have
        \begin{align}
            \log p(y\vert x) & = \ELBO{q}{p}  + \KL{q}{p}.
        \end{align}
        Therefore, we have 
        \begin{align}
            \min_u \min_v  \KL{\qab}p \le \min_{v,w}  \KL{\qb}p 
        \end{align}
        From \cref*{thm: branch q better than q for hbd}, we get the desired result. 
        \begin{align}
            \min_u \min_v  \KL{\qab}p \le \min_{v,w}  \KL{\qb}p\le \min_{\phi}  \KL{\qj}p 
        \end{align}
\end{proof}

\section{Derivation for Branch Gaussian}
\label{app sec: derivation of branch gaussian}

Let $\qj(\theta, z)  = \N ((\theta, z) \vert \mu, \Sigma)$ be the joint Gaussian approximation as in \cref{cor: branch gaussian}. Further, let $(\mu, \Sigma)$ be defined as 
\begin{align}
    \mu = \begin{bmatrix}
        \mu_{\theta}\\
        \mu_{z_1}\\
        \vdots\\
        \mu_{z_N}
    \end{bmatrix}\quad \text{and } \quad  
    \Sigma = \begin{bmatrix}
        \Sigma_\theta & \Sigma_{\theta z_1} & \dots & \Sigma_{\theta z_N}\\
        \Sigma_{\theta z_1}^\top & \Sigma_{z_1 } & \dots & \Sigma_{z_1 z_N}\\
        \vdots & \vdots & \ddots & \vdots\\
        \Sigma_{\theta z_n}^\top & \Sigma_{z_1 z_N}^\top & \dots & \Sigma_{z_N }\\
    \end{bmatrix}.
\end{align}
Then,  from the properties of the multivariate Gaussian \cite{IMM2012-03274}
\begin{align}
    \qj(z_i \vert \theta) &= \N (z_i \vert \mu_{z_i \vert \theta}, \Sigma_{z_i \vert \theta}) \text{, where}\\
    \mu_{z_i \vert \theta} & = \mu_{z_i} + \Sigma_{\theta z_i}^\top \Sigma_{\theta}^{-1}(\theta - \mu_\theta) \text{, and }\\
    \Sigma_{z_i \vert \theta} & = \Sigma_{z_i z_i} - \Sigma_{\theta z_i}^\top \Sigma_{\theta}^{-1} \Sigma_{\theta z_i}.
\end{align}
Now, to parameterize a corresponding $\qb$, we use the $(\mu_i, \Sigma_i, A_i)$, such that, 

\begin{align}
    \qb(z_i \vert \theta) &= \N (z_i \vert \mu_i + A_i \theta, \Sigma_i).
\end{align}


\section{Experimental Details}
\label{app sec: experimental details}
\paragraph{Architectural Details}

\begin{wraptable}{r}{0.35\textwidth}
    \vspace{-5pt}
    \caption{
        Architecture details for $\net_u$. 
        Each fully-connected layer is followed by leaky-ReLU baring the last layer.  
        }
        \centering
        \label{tab: arch details}
    \resizebox*{0.3\textwidth}{!}{
        \begin{tabular}[!h]{@{} ll @{}}
            \toprule
            Network & Layer Skeleton\\
            \midrule
            \texttt{feat-net} & 64, 64, 64, 128\\
            \texttt{param-net} & 256, 256, 256\\
            \bottomrule
        \end{tabular}
        }
    \vspace{-10pt}
    \end{wraptable}
We use the architecture as reported in \cref{tab: arch details} for all our amortized approaches. In addition to 
using $e_j$ as detailed in  \cref{fig: featnet}, we concatenate the elementwise square before sending it to \texttt{param-net}. 
Thus, the input to \texttt{param-net} is not 128 dimensional but 256 dimensional. Further, we use mean as the \texttt{pool} function. 

\paragraph{Compute Resources}
We use JAX \cite{jax2018github} to implement our methods. We trained using Nvidia 2080ti-12GB. All methods finished training within 4 hours. Branch approaches were 
at an average twice as fast as amortized variants.  

\paragraph{Step-size drop}
We use Adam \cite{kingma2014adam} for training with an initial step-size of 0.001 (and default values for other hyperparameters.) 
In preliminary experiments, we found that dropping the step-size improves the performance. 
Starting from 0.001, we drop the step to one-tenth of it's value after a predetermined number of steps. 
For small scale experiments, we drop a total of three times after every 50,000 iterations (we train for 200,000 iterations.)
For moderate and large scale models, we drop once after 100,000 iterations. 

\subsection{Movielens}
\label{app sec: movielens details}
\begin{wraptable}{r}{0.4\textwidth}
    \vspace{-10pt}
    \caption{\small Metrics used for evaluation. We use K = 10,000 samples from the posterior. Here, $(z^{k}, \theta^{k}) \sim q(\rz, \rtheta \vert x^{\mathrm{train}}, y^{\mathrm{train}})$.}
        \centering
        \label{tab: metrics}
    \resizebox*{0.4\textwidth}{!}{
        \begin{tabular}[!h]{@{} ll @{}}
            \toprule
            Metric & Expression\\
            \midrule
            Test likelihood & $\log \frac{1}{K}\sum_{k}  p(y^{\mathrm{test}}\vert x^{\mathrm{test}}, z^k, \theta^k)$ \\[1.5mm]
            Train likelihood & $\log \frac{1}{K}\sum_{k}  \frac{p(y^{\mathrm{train}}, z^k, \theta^k\vert x^{\mathrm{train}})}{q(z^k, \theta^k \vert x^{\mathrm{train}}, y^{\mathrm{train}})}$ \\[1.5mm]
            Train ELBO & $\frac{1}{K}\sum_{k} \log  \frac{p(y^{\mathrm{train}}, z^k, \theta^k\vert x^{\mathrm{train}})}{q(z^k, \theta^k \vert x^{\mathrm{train}}, y^{\mathrm{train}})}$ \\
            \bottomrule
        \end{tabular}
        }
    \vspace{-15pt}
\end{wraptable}

\paragraph{Feature Dimensions}
We reduce the movie feature dimensionality to 10 using PCA. This is done with branch approaches in focus as the number of 
features for dense branch Gaussian scale as $\mathcal{O}(ND^3)$, where $D$ is the dimensionality of the movie features. Note, 
that the number of features for amortized approaches is independent of $N$ allowing for better scalability. 

\paragraph{Metrics}
We use three metrics for performance evaluation---test likelihood, train likelihood, and train ELBO. Details of the expressions are 
presented in \cref{tab: metrics}. We draw a batch of fresh 10,000 samples from the posterior to estimate each metric. Of course, the
evaluated expressions are just approximation to the true value. In \cref{tab: extended movielens a} and \cref{tab: extended movielens b} 
we present the extended results. In \cref{tab: extended movielens b} we present the same values but normalized by the number of ratings in the dataset.
\paragraph{Preprocess}

Movielens25M originally uses a 5 point ratings system. To get binary ratings, we map ratings greater than 3 points to 1 and less than and equal to 3 to 0.
\begin{table}
    \caption{This table has the extended results for the Movielens25M Dataset. All values are in nats. Higher is better. }
    \label{tab: extended movielens a}
    \resizebox{\textwidth}{!}{
        \begin{tabular}{@{} llrllrllrll @{}}
        \toprule
                 & {} & \multicolumn{3}{l}{Test LL} & \multicolumn{3}{l}{Train LL} & \multicolumn{3}{l}{Train ELBO} \\
        $\approx$ \# ratings &  &    \multicolumn{1}{c}{2.5K} &      \multicolumn{1}{c}{180K} &         \multicolumn{1}{c}{18M} &     \multicolumn{1}{c}{2.5K} &      \multicolumn{1}{c}{180K} &         \multicolumn{1}{c}{18M} &       \multicolumn{1}{c}{2.5K} &        \multicolumn{1}{c}{180K} &         \multicolumn{1}{c}{18M} \\
        Methods & {} &         &           &             &          &           &             &            &             &             \\
        \midrule
        Dense & $\qj$ & -166.37 &           &             & -1373.97 &           &             &   -1572.31 &             &             \\
                 & $\qb$ & -166.66 & -11054.43 & -1.3046e+06 & -1374.20 & -95731.42 & -1.0315e+07 &   -1572.39 & -1.0368e+05 & -1.1413e+07 \\
                 & $\qab$ & -166.64 & -10976.38 & -1.1476e+06 & -1374.27 & -95980.37 & -1.0027e+07 &   -1572.45 & -1.0352e+05 & -1.0665e+07 \\[2.5mm]
        Block & $\qj$ & -167.36 &           &             & -1375.56 &           &             &   -1579.04 &             &             \\
        Diagonal& $\qb$ & -166.97 & -10987.17 & -1.2538e+06 & -1375.71 & -95891.42 & -1.0399e+07 &   -1579.05 & -1.0350e+05 & -1.1078e+07 \\
                 & $\qab$ & -166.96 & -10975.96 & -1.1484e+06 & -1375.71 & -95962.56 & -1.0027e+07 &   -1579.06 & -1.0353e+05 & -1.0665e+07 \\[2.5mm]
        Diagonal & $\qj$ & -167.39 &           &             & -1377.25 &           &             &   -1592.59 &             &             \\
                 & $\qb$ & -167.31 & -10977.95 & -1.2713e+06 & -1377.19 & -96414.40 & -1.0709e+07 &   -1592.64 & -1.0428e+05 & -1.1325e+07 \\
                 & $\qab$ & -167.29 & -10980.75 & -1.1497e+06 & -1377.20 & -96467.88 & -1.0068e+07 &   -1592.64 & -1.0430e+05 & -1.0736e+07 \\
        \bottomrule
        \end{tabular}
        }
\end{table}
\begin{table}
    \caption{This table has the extended results for the Movilens25M Dataset. It has the same results as in \cref{tab: extended movielens a}; however, 
    the values are divided by the number of ratings. }
    \label{tab: extended movielens b}
    \resizebox{\textwidth}{!}{\begin{tabular}{@{} llrllrllrll @{}}
        \toprule
                 & {} & \multicolumn{3}{l}{Test LL} & \multicolumn{3}{l}{Train LL} & \multicolumn{3}{l}{Train ELBO} \\
                 & $\approx$ \# ratings &    \multicolumn{1}{c}{2.5K} &    \multicolumn{1}{c}{180K} &     \multicolumn{1}{c}{18M} &     \multicolumn{1}{c}{2.5K} &    \multicolumn{1}{c}{180K} &     \multicolumn{1}{c}{18M} &       \multicolumn{1}{c}{2.5K} &    \multicolumn{1}{c}{180K} &     \multicolumn{1}{c}{18M }\\
        Methods & {} &         &         &         &          &         &         &            &         &         \\
        \midrule
        Dense & $\qj$ & -0.5717 &         &         &  -0.5108 &         &         &    -0.5845 &         &         \\
                 & $\qb$ & -0.5727 & -0.5640 & -0.6486 &  -0.5109 & -0.5224 & -0.5492 &    -0.5845 & -0.5658 & -0.6077 \\
                 & $\qab$ & -0.5726 & -0.5600 & -0.5705 &  -0.5109 & -0.5238 & -0.5339 &    -0.5846 & -0.5649 & -0.5678 \\ [2.5mm]
        Block & $\qj$ & -0.5751 &         &         &  -0.5114 &         &         &    -0.5870 &         &         \\
        Diagonal & $\qb$ & -0.5738 & -0.5606 & -0.6233 &  -0.5114 & -0.5233 & -0.5537 &    -0.5870 & -0.5648 & -0.5898 \\
                 & $\qab$ & -0.5738 & -0.5600 & -0.5709 &  -0.5114 & -0.5237 & -0.5339 &    -0.5870 & -0.5650 & -0.5678 \\[2.5mm]
        Diagonal & $\qj$ & -0.5752 &         &         &  -0.5120 &         &         &    -0.5920 &         &         \\
                 & $\qb$ & -0.5749 & -0.5601 & -0.6320 &  -0.5120 & -0.5261 & -0.5702 &    -0.5921 & -0.5691 & -0.6029 \\
                 & $\qab$ & -0.5749 & -0.5602 & -0.5716 &  -0.5120 & -0.5264 & -0.5360 &    -0.5921 & -0.5691 & -0.5716 \\
        \bottomrule
        \end{tabular}
        }
\end{table}
\subsection{Synthetic problem}
\label{app sec: synthetic details}
\paragraph{Details of the model}

We use the hierarchical regression model  
\begin{align*}
    p(\theta, z, y \vert x) &= \mathcal{N}(\theta \vert 0, I) \prod_{i=1}^{N} \mathcal{N}(z_i \vert \theta, I) \prod_{j=1}^{n_i}\mathcal{N}(y_{ij} \vert x_{ij}^\top z_i, 1)
\end{align*}
for synthetic experiments. For simplicity, we use $n_i = 100$ for all $i$; we vary $N$ to create different scale variants---we use $N = 10$ for small scale, $N = 1000$ for moderate scale, 
and $N = 100000$ for large scale experiments; we set $x_{ij} \in \R^{10}$ and thus $\theta \in \R^{10}$ and $ z_i \in \R^{10}$; $y_{ij} \in \R$.

\paragraph{Expression for posterior}
\begin{align*}
    p(\theta \vert x, y) &= \mathsmaller{\displaystyle  \N \Bigg(\left[I_D + \sum_i x_i^\top (I_{M_i} + x_i x_i^\top)^{-1} x_i\right]^{-1} \left[\sum_i x_i^\top (I_M + x_i x_i^\top)^{(-1)}y_i\right],} \\
                         & \mathsmaller{\displaystyle  \phantom{Something something something ething}  \left[ I_D + \sum_i x_i^\top (I_M + x_i x_i^\top)^{-1} x_i\right]^{-1}\Bigg)}\\
    p(z_i \vert \theta, x_i, y_i) &= \mathsmaller{\N ([I_D +  x_i^\top x_i]^{-1} [x_i^\top y_i + \theta], [I_D +  x_i^\top x_i]^{-1})}\\
\end{align*}
\paragraph{Expression for marginal likelihood}
\begin{align*}
    p(y|x) 
         & =\mathcal{N}\left(0,\begin{bmatrix}I_{M}+2x_{1}x_{1}^{\top} & \dots & x_{n}x_{1}^{\top}\\
        \vdots & \ddots & \vdots\\
        x_{n}x_{1}^{\top} & \dots & I_{M}+2x_{n}x_{n}^{\top}
        \end{bmatrix}\right)\\
\end{align*}



\begin{figure*}[!h]
    \centering
    \resizebox*{\textwidth}{!}{
        \begin{subfigure}[b]{\textwidth}
            \centering
            \includegraphics[trim=0 0 0 0,clip,width=\textwidth]{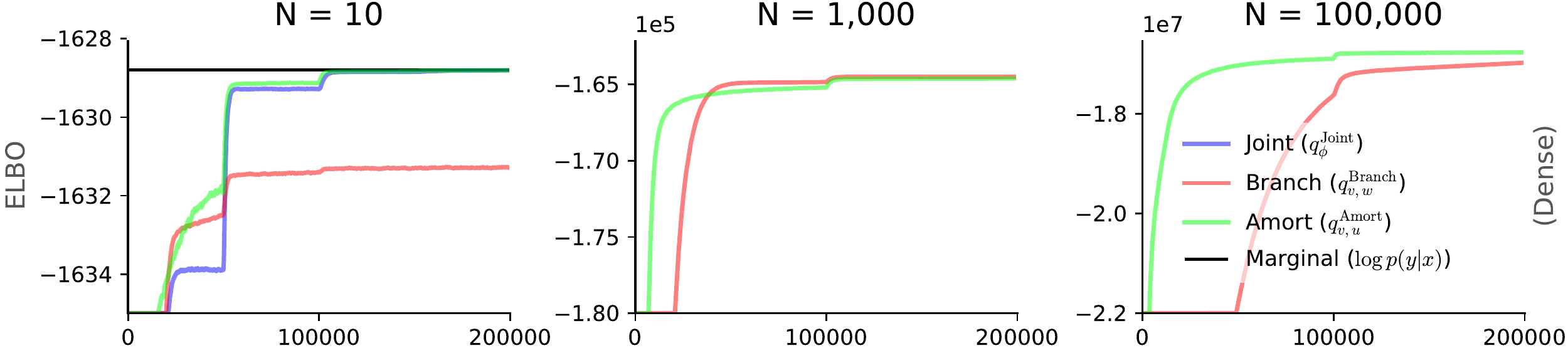}
        \end{subfigure}
    }
    \resizebox*{\textwidth}{!}{
            \begin{subfigure}[b]{\textwidth}
                \centering
                \includegraphics[trim=0 0 0 0,clip,width=\textwidth]{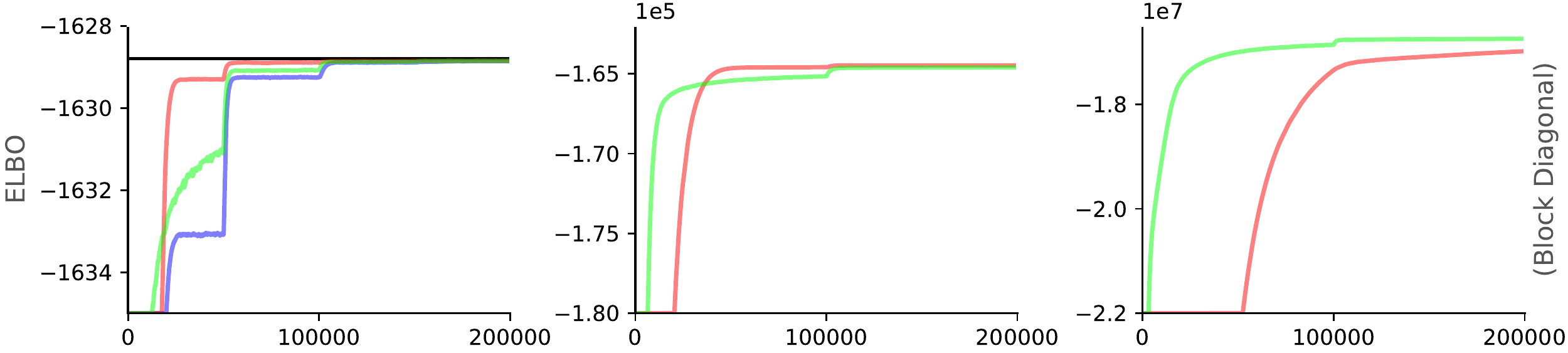}
            \end{subfigure}
    }
    \resizebox*{\textwidth}{!}{
                \begin{subfigure}[b]{\textwidth}
                    \centering
                    \includegraphics[trim=0 0 0 0,clip,width=\textwidth]{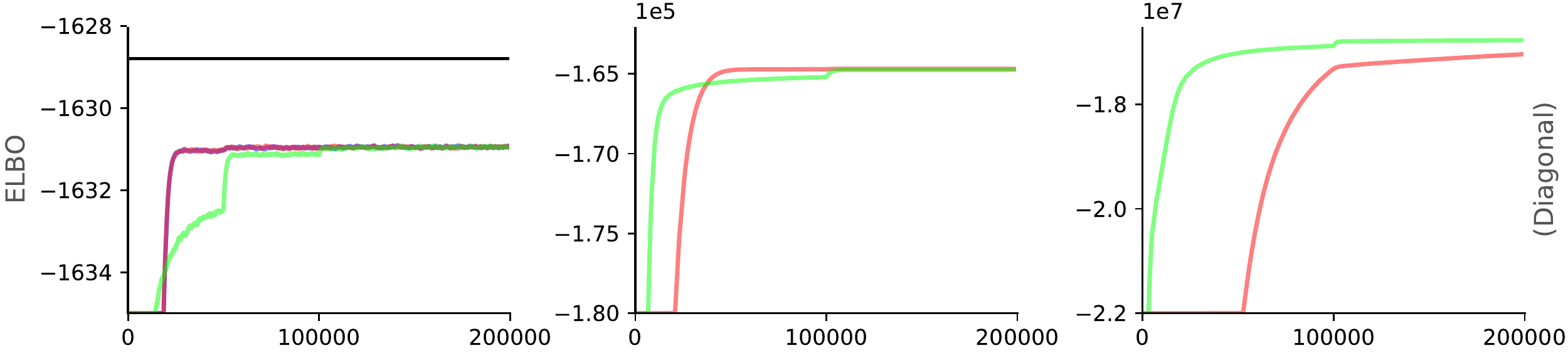}
                \end{subfigure}
    }
       \caption{
        Training ELBO trace for the synthetic problem.
        Top to bottom:  dense, 
        block diagonal, and diagonal Gaussian  
        (for each, we have $\qj$, $\qb$, and $\qab$ method.)
        Left to right: small, moderate, and large scale of the synthetic problem.   
        For clarity, we plot the exponential moving average
        of the training ELBO trace with a smoothing value of 0.001. 
        For the small setting, we also plot the true log-marginal $\log p(y \vert x)$ for reference (black horizontal line): 
        ELBO for dense approach is exactly same as the log-marginal, it's slightly lower for block, and is much less for the diagonal (see first column.) 
        Note, calculating the log-marginal was computationally prohibitive for the moderate and large setting.}
        \label{fig: synthetic results}
\end{figure*}

%% file: scalable_vi.bbl
\begin{thebibliography}{37}
\providecommand{\natexlab}[1]{#1}
\providecommand{\url}[1]{\texttt{#1}}
\expandafter\ifx\csname urlstyle\endcsname\relax
  \providecommand{\doi}[1]{doi: #1}\else
  \providecommand{\doi}{doi: \begingroup \urlstyle{rm}\Url}\fi

\bibitem[Agrawal et~al.(2020)Agrawal, Sheldon, and Domke]{aagrawal2020}
Abhinav Agrawal, Daniel~R. Sheldon, and Justin Domke.
\newblock Advances in black-box {VI:} normalizing flows, importance weighting,
  and optimization.
\newblock In \emph{NeurIPS 2020}, 2020.

\bibitem[Ambrogioni et~al.(2021)Ambrogioni, Lin, Fertig, Vikram, Hinne, Moore,
  and Gerven]{ambrogioni2021automatic}
Luca Ambrogioni, Kate Lin, Emily Fertig, Sharad Vikram, Max Hinne, Dave Moore,
  and Marcel Gerven.
\newblock Automatic structured variational inference.
\newblock In \emph{International Conference on Artificial Intelligence and
  Statistics}, pages 676--684. PMLR, 2021.

\bibitem[Blei(2012)]{blei2012probabilistic}
David~M Blei.
\newblock Probabilistic topic models.
\newblock \emph{Communications of the ACM}, 55\penalty0 (4):\penalty0 77--84,
  2012.

\bibitem[Blei et~al.(2003)Blei, Ng, and Jordan]{blei2003latent}
David~M Blei, Andrew~Y Ng, and Michael~I Jordan.
\newblock Latent dirichlet allocation.
\newblock \emph{the Journal of machine Learning research}, 3:\penalty0
  993--1022, 2003.

\bibitem[Bouchacourt et~al.(2018)Bouchacourt, Tomioka, and
  Nowozin]{bouchacourt2018multi}
Diane Bouchacourt, Ryota Tomioka, and Sebastian Nowozin.
\newblock Multi-level variational autoencoder: Learning disentangled
  representations from grouped observations.
\newblock In \emph{Thirty-Second AAAI Conference on Artificial Intelligence},
  2018.

\bibitem[Bradbury et~al.(2018)Bradbury, Frostig, Hawkins, Johnson, Leary,
  Maclaurin, Necula, Paszke, Vander{P}las, Wanderman-{M}ilne, and
  Zhang]{jax2018github}
James Bradbury, Roy Frostig, Peter Hawkins, Matthew~James Johnson, Chris Leary,
  Dougal Maclaurin, George Necula, Adam Paszke, Jake Vander{P}las, Skye
  Wanderman-{M}ilne, and Qiao Zhang.
\newblock {JAX}: composable transformations of {P}ython+{N}um{P}y programs,
  2018.
\newblock URL \url{http://github.com/google/jax}.

\bibitem[Cressie et~al.(2009)Cressie, Calder, Clark, Hoef, and
  Wikle]{cressie2009accounting}
Noel Cressie, Catherine~A Calder, James~S Clark, Jay M~Ver Hoef, and
  Christopher~K Wikle.
\newblock Accounting for uncertainty in ecological analysis: the strengths and
  limitations of hierarchical statistical modeling.
\newblock \emph{Ecological Applications}, 19\penalty0 (3):\penalty0 553--570,
  2009.

\bibitem[Domke(2020)]{domke2020provable}
Justin Domke.
\newblock Provable smoothness guarantees for black-box variational inference.
\newblock In \emph{International Conference on Machine Learning}, pages
  2587--2596. PMLR, 2020.

\bibitem[Edwards and Storkey(2016)]{edwards2016towards}
Harrison Edwards and Amos Storkey.
\newblock Towards a neural statistician.
\newblock \emph{arXiv preprint arXiv:1606.02185}, 2016.

\bibitem[Gelman(2009)]{gelman2009red}
Andrew Gelman.
\newblock \emph{Red state, blue state, rich state, poor state: Why Americans
  vote the way they do-expanded edition}.
\newblock Princeton University Press, 2009.

\bibitem[Gelman and Hill(2006)]{gelman2006data}
Andrew Gelman and Jennifer Hill.
\newblock \emph{Data analysis using regression and multilevel/hierarchical
  models}.
\newblock Cambridge university press, 2006.

\bibitem[Gelman and Vehtari(2020)]{gelman2020most}
Andrew Gelman and Aki Vehtari.
\newblock What are the most important statistical ideas of the past 50 years?
\newblock \emph{arXiv preprint arXiv:2012.00174}, 2020.

\bibitem[Harper and Konstan(2015)]{harper2015movielens}
F~Maxwell Harper and Joseph~A Konstan.
\newblock The movielens datasets: History and context.
\newblock \emph{Acm transactions on interactive intelligent systems (tiis)},
  5\penalty0 (4):\penalty0 1--19, 2015.

\bibitem[Hoffman and Blei(2015)]{hoffman2015ssvi}
Matthew~D. Hoffman and David~M. Blei.
\newblock Stochastic structured variational inference.
\newblock In Guy Lebanon and S.~V.~N. Vishwanathan, editors, \emph{Proceedings
  of the Eighteenth International Conference on Artificial Intelligence and
  Statistics, {AISTATS} 2015, San Diego, California, USA, May 9-12, 2015},
  volume~38 of \emph{{JMLR} Workshop and Conference Proceedings}. JMLR.org,
  2015.
\newblock URL \url{http://proceedings.mlr.press/v38/hoffman15.html}.

\bibitem[Hoffman et~al.(2013)Hoffman, Blei, Wang, and
  Paisley]{hoffman2013stochastic}
Matthew~D Hoffman, David~M Blei, Chong Wang, and John Paisley.
\newblock Stochastic variational inference.
\newblock \emph{Journal of Machine Learning Research}, 14\penalty0 (5), 2013.

\bibitem[Ilse et~al.(2019)Ilse, Tomczak, Louizos, and Welling]{ilse2019diva}
Maximilian Ilse, Jakub~M Tomczak, Christos Louizos, and Max Welling.
\newblock Diva: Domain invariant variational autoencoders. arxiv e-prints,
  article.
\newblock \emph{arXiv preprint arXiv:1905.10427}, 2019.

\bibitem[Johnson et~al.(2016)Johnson, Duvenaud, Wiltschko, Adams, and
  Datta]{johnson2016composing}
Matthew~J Johnson, David~K Duvenaud, Alex Wiltschko, Ryan~P Adams, and
  Sandeep~R Datta.
\newblock Composing graphical models with neural networks for structured
  representations and fast inference.
\newblock \emph{Advances in neural information processing systems},
  29:\penalty0 2946--2954, 2016.

\bibitem[Kingma and Ba(2015)]{kingma2014adam}
Diederik~P Kingma and Jimmy Ba.
\newblock Adam: A method for stochastic optimization.
\newblock 2015.

\bibitem[Kingma and Welling(2013)]{kingma2013auto}
Diederik~P Kingma and Max Welling.
\newblock Auto-encoding variational bayes.
\newblock \emph{arXiv preprint arXiv:1312.6114}, 2013.

\bibitem[Kucukelbir et~al.(2017)Kucukelbir, Tran, Ranganath, Gelman, and
  Blei]{kucukelbir2017automatic}
Alp Kucukelbir, Dustin Tran, Rajesh Ranganath, Andrew Gelman, and David~M Blei.
\newblock Automatic differentiation variational inference.
\newblock \emph{The Journal of Machine Learning Research}, 18\penalty0
  (1):\penalty0 430--474, 2017.

\bibitem[Lafferty and Blei(2006)]{lafferty2006correlated}
John~D Lafferty and David~M Blei.
\newblock Correlated topic models.
\newblock \emph{Advances in neural information processing systems},
  18:\penalty0 147--154, 2006.

\bibitem[Lawson(2008)]{lawson2008bayesian}
Andrew~B Lawson.
\newblock \emph{Bayesian disease mapping: hierarchical modeling in spatial
  epidemiology}.
\newblock Chapman and Hall/CRC, 2008.

\bibitem[Lax and Phillips(2012)]{lax2012democratic}
Jeffrey~R Lax and Justin~H Phillips.
\newblock The democratic deficit in the states.
\newblock \emph{American Journal of Political Science}, 56\penalty0
  (1):\penalty0 148--166, 2012.

\bibitem[LeCun et~al.(2012)LeCun, Bottou, Orr, and
  M{\"u}ller]{lecun2012efficient}
Yann~A LeCun, L{\'e}on Bottou, Genevieve~B Orr, and Klaus-Robert M{\"u}ller.
\newblock Efficient backprop.
\newblock In \emph{Neural networks: Tricks of the trade}, pages 9--48.
  Springer, 2012.

\bibitem[Lim and Teh(2007)]{lim2007variational}
Yew~Jin Lim and Yee~Whye Teh.
\newblock Variational bayesian approach to movie rating prediction.
\newblock In \emph{Proceedings of KDD cup and workshop}, volume~7, pages
  15--21. Citeseer, 2007.

\bibitem[Papamakarios et~al.(2019)Papamakarios, Nalisnick, Rezende, Mohamed,
  and Lakshminarayanan]{papamakarios2019normalizing}
George Papamakarios, Eric Nalisnick, Danilo~Jimenez Rezende, Shakir Mohamed,
  and Balaji Lakshminarayanan.
\newblock Normalizing flows for probabilistic modeling and inference.
\newblock \emph{arXiv preprint arXiv:1912.02762}, 2019.

\bibitem[Petersen and Pedersen(2012)]{IMM2012-03274}
K.~B. Petersen and M.~S. Pedersen.
\newblock The matrix cookbook, nov 2012.
\newblock URL \url{http://www2.compute.dtu.dk/pubdb/pubs/3274-full.html}.
\newblock Version 20121115.

\bibitem[Ranganath et~al.(2014)Ranganath, Gerrish, and Blei]{ranganath14}
Rajesh Ranganath, Sean Gerrish, and David Blei.
\newblock {Black Box Variational Inference}.
\newblock In \emph{AISTATS}, 2014.

\bibitem[Rezende et~al.(2014)Rezende, Mohamed, and
  Wierstra]{pmlr-v32-rezende14}
Danilo~Jimenez Rezende, Shakir Mohamed, and Daan Wierstra.
\newblock Stochastic backpropagation and approximate inference in deep
  generative models.
\newblock In \emph{ICML}, 2014.

\bibitem[Roeder et~al.(2017)Roeder, Wu, and Duvenaud]{roeder2017sticking}
Geoffrey Roeder, Yuhuai Wu, and David Duvenaud.
\newblock Sticking the landing: Simple, lower-variance gradient estimators for
  variational inference.
\newblock In \emph{NIPS}, 2017.

\bibitem[Salakhutdinov and Mnih(2008)]{salakhutdinov2008bayesian}
Ruslan Salakhutdinov and Andriy Mnih.
\newblock Bayesian probabilistic matrix factorization using markov chain monte
  carlo.
\newblock In \emph{Proceedings of the 25th international conference on Machine
  learning}, pages 880--887, 2008.

\bibitem[Sheth and Khardon(2016)]{sheth2016monte}
Rishit Sheth and Roni Khardon.
\newblock Monte carlo structured svi for two-level non-conjugate models.
\newblock \emph{arXiv preprint arXiv:1612.03957}, 2016.

\bibitem[Tipping and Bishop(1999)]{tipping1999probabilistic}
Michael~E Tipping and Christopher~M Bishop.
\newblock Probabilistic principal component analysis.
\newblock \emph{Journal of the Royal Statistical Society: Series B (Statistical
  Methodology)}, 61\penalty0 (3):\penalty0 611--622, 1999.

\bibitem[Titsias and L{\'a}zaro-Gredilla(2014)]{titsias2014doubly}
Michalis Titsias and Miguel L{\'a}zaro-Gredilla.
\newblock Doubly stochastic variational bayes for non-conjugate inference.
\newblock In \emph{International conference on machine learning}, pages
  1971--1979. PMLR, 2014.

\bibitem[Vallerand(1997)]{vallerand1997toward}
Robert~J Vallerand.
\newblock Toward a hierarchical model of intrinsic and extrinsic motivation.
\newblock \emph{Advances in experimental social psychology}, 29:\penalty0
  271--360, 1997.

\bibitem[Vig et~al.(2012)Vig, Sen, and Riedl]{vig2012tag}
Jesse Vig, Shilad Sen, and John Riedl.
\newblock The tag genome: Encoding community knowledge to support novel
  interaction.
\newblock \emph{ACM Transactions on Interactive Intelligent Systems (TiiS)},
  2\penalty0 (3):\penalty0 1--44, 2012.

\bibitem[Zaheer et~al.(2017)Zaheer, Kottur, Ravanbakhsh, Poczos, Salakhutdinov,
  and Smola]{deepsets}
Manzil Zaheer, Satwik Kottur, Siamak Ravanbakhsh, Barnabas Poczos, Russ~R
  Salakhutdinov, and Alexander~J Smola.
\newblock Deep sets.
\newblock In I.~Guyon, U.~V. Luxburg, S.~Bengio, H.~Wallach, R.~Fergus,
  S.~Vishwanathan, and R.~Garnett, editors, \emph{Advances in Neural
  Information Processing Systems}, volume~30. Curran Associates, Inc., 2017.
\newblock URL
  \url{https://proceedings.neurips.cc/paper/2017/file/f22e4747da1aa27e363d86d40ff442fe-Paper.pdf}.

\end{thebibliography}
